\def\year{2018}\relax
\documentclass[letterpaper]{article} 
\usepackage{aaai18}  
\usepackage{times}  
\usepackage{helvet}  
\usepackage{courier}  
\usepackage{url}  
\usepackage{graphicx}  
\usepackage{booktabs}       
\usepackage{nicefrac}       

\usepackage{amsmath, amssymb, latexsym}
\usepackage{algorithm, algorithmic}
\usepackage{subfigure, wrapfig}
\usepackage[english]{babel}

\newcommand{\tabincell}[2]{\begin{tabular}{@{} #1 @{}} #2 \end{tabular}}

\newcommand{\citet}[1]{\citeauthor{#1}~\shortcite{#1}}
\newcommand{\citep}[1]{\citeauthor{#1},~\citeyear{#1}}
\newcommand{\citenop}[1]{\citeauthor{#1}~\citeyear{#1}}

\newcommand{\re}{\mathbb{R}}
\newcommand{\sph}{\mathbb{S}}
\newcommand{\expect}{\mathbb{E}}
\newcommand{\stiefel}{\mathbb{M}}

\newcommand{\stein}{\mathcal{A}}
\newcommand{\cont}{\mathcal{C}}

\newcommand{\hilb}{\mathcal{H}}
\newcommand{\fisher}{\mathcal{I}}
\newcommand{\obj}{\mathcal{J}}

\newcommand{\mf}{\mathcal{M}}
\newcommand{\normal}{\mathcal{N}}
\newcommand{\order}{\mathcal{O}}
\newcommand{\tg}{\mathcal{T}}

\newcommand{\ud}{\mathrm{d}}
\newcommand{\Bern}{\mathrm{Bern}}
\newcommand{\dir}{\mathrm{Dir}}
\renewcommand{\div}{\mathrm{div}}
\newcommand{\Exp}{\mathrm{Exp}}
\newcommand{\grad}{\mathrm{grad}\,}
\newcommand{\gradn}{\mathrm{grad}}
\newcommand{\jac}{\mathrm{Jac}\,}
\newcommand{\kl}{\mathrm{KL}}

\newcommand{\prob}{\mathrm{Prob}}
\newcommand{\tr}{\mathrm{tr}}
\newcommand{\vmf}{\mathrm{vMF}}

\newcommand{\const}{\mathrm{const}}

\newcommand{\subtg}{\mathfrak{X}}

\newcommand{\pdf}{p.d.f.~}
\newcommand{\wrt}{w.r.t.~}
\newcommand{\trs}{^{\top}}

\newcommand{\defas}{:=}

\usepackage{amsthm}
\theoremstyle{plain}
\newtheorem{thm}{Theorem}
\newtheorem{lem}[thm]{Lemma}
\newtheorem{prop}[thm]{Proposition}
\newtheorem{cor}[thm]{Corollary}
\theoremstyle{remark}

\begin{document}
%
\title{Riemannian Stein Variational Gradient Descent for Bayesian Inference}
\author{
  Chang Liu, \, Jun Zhu\thanks{corresponding author.}\\
    Dept. of Comp. Sci. \& Tech., TNList Lab; Center for Bio-Inspired Computing Research \\
    State Key Lab for Intell. Tech. \& Systems, Tsinghua University, Beijing, China \\
  \texttt{chang-li14@mails.tsinghua.edu.cn;~dcszj@tsinghua.edu.cn}
}
\maketitle
\begin{abstract}
  We develop Riemannian Stein Variational Gradient Descent (RSVGD), a Bayesian inference method that generalizes Stein Variational Gradient Descent (SVGD) to Riemann manifold.
  The benefits are two-folds: (i) for inference tasks in Euclidean spaces, RSVGD has the advantage over SVGD of utilizing information geometry, and (ii) for inference tasks on Riemann manifolds, RSVGD brings the unique advantages of SVGD to the Riemannian world.
  To appropriately transfer to Riemann manifolds, we conceive novel and non-trivial techniques for RSVGD, which are required by the intrinsically different characteristics of general Riemann manifolds from Euclidean spaces.
  We also discover Riemannian Stein's Identity and Riemannian Kernelized Stein Discrepancy.
  Experimental results show the advantages over SVGD of exploring distribution geometry and the advantages of particle-efficiency, iteration-effectiveness and approximation flexibility over other inference methods on Riemann manifolds.
\end{abstract}

\section{Introduction}

Bayesian inference is the central task for learning a Bayesian model to extract knowledge from data.
The task is to estimate the posterior distribution of latent variables of the model given observed data.
It has been in the focus of machine learning for decades, with quite a lot of methods emerging.
Variational inference methods (VIs) aim to approximate the posterior by a tractable variational distribution.
Traditional VIs typically use a statistical model, usually a parametric distribution family, as the variational distribution, and we call them model-based VIs (M-VIs).
They cast the inference problem as an optimization problem, which can be efficiently solved by various techniques.
However, due to the restricted coverage of the chosen distribution family (e.g. the mean-field form), there would be a gap blocking the approximation from getting any closer. 
Monte Carlo methods (MCs) estimate the posterior by directly drawing samples from it.
Asymptotically accurate as they are, their performance for finite samples is not guaranteed and usually take effect slowly, especially for the widely applicable thus commonly used Markov Chain MCs (MCMCs), due to the positive autocorrelation of their samples.

Recently, a set of particle-based VIs (P-VIs) have been proposed.
P-VIs use a certain number of samples, or particles, to represent the variational distribution, and update the particles by solving an optimization problem.
Similar to MCs, this non-parametric particle form gives them great flexibility to reduce the gap of M-VIs, and beyond MCs, the optimization-based update rule makes them effective in iteration: every iteration is guaranteed to make progress.
Although there are convergence analyses for some particular MCMCs from non-stationary to stationary, the principle of a general MCMC only guarantees that the sampler will remain in the stationary distribution.
Moreover, MCs usually require a large sample size to take effect, while P-VIs achieve similar performance with much fewer particles, since their principle aims at finite sample performance.
This particle-efficiency would save the storage of the inference result, and reduce the time for tasks afterwards such as test and prediction.
Table~\ref{tab:cmp} presents a comparison of the three kinds of inference methods.

Stein variational gradient descent (SVGD) \cite{liu2016stein} is an outstanding example of P-VIs.
SVGD updates particles by imposing a continuous-time dynamics on them that leads the variational distribution to evolve towards the posterior.
Although the dynamics is restricted in a kernel-related space for tractability, the theory of SVGD makes no assumption on the variational distribution, indicating the best flexibility.
SVGD has been applied to develop advanced inference methods (\citenop{wang2016learning}; \citenop{pu2017stein}) as well as reinforcement learning (\citenop{liu2017stein}; \citenop{haarnoja2017reinforcement}).
Other instances of P-VIs include normalizing flows (NF) \cite{rezende2015variational} and particle mirror descent (PMD) \cite{dai2016provable}.
NF uses a series of invertible transformations to adjust particles from a simple tractable distribution to fit the posterior.
However, the invertibility requirement restricts its flexibility.
PMD adopts the mirror descent method to formulate an optimization problem and uses a weighted kernel density estimator for the variational distribution, which is still a restricting assumption.

\begin{table}[t]
  \caption{A comparison of three kinds of inference methods.}
  \label{tab:cmp}
  \centering
  \begin{tabular}{cccc}
    \toprule
	Methods & M-VIs & MCs & P-VIs \\
	\midrule\midrule
	\tabincell{c}{Asymptotic \\ Accuracy} & No & Yes & Promising \\
	\midrule
	\tabincell{c}{Approximation \\ Flexibility} & Limited & Unlimited & Unlimited \\
	\midrule
	\tabincell{c}{Iteration- \\ Effectiveness} &  Yes & Weak & Strong \\
	\midrule
	\tabincell{c}{Particle- \\ Efficiency} & \tabincell{c}{(do not\\apply)} & Weak & Strong \\
    \bottomrule
  \end{tabular}
\end{table}

Another issue of Bayesian inference is the collaboration with Riemann manifold.
This consideration bears its importance in two ways:
(i) the posterior of some models itself is a distribution on a given Riemann manifold, e.g. the spherical admixture model (SAM) \cite{reisinger2010spherical} has its posterior on hyperspheres;
(ii) latent variables in a Bayesian model is a natural coordinate system (c.s.) of the Riemann manifold of likelihood distribution, so we can also conduct Bayesian inference on the distribution manifold with improved efficiency with the help of information geometry (\citenop{amari2007methods}; \citenop{amari2016information}).
Much progress has been made recently for both Riemannian considerations.
In the spirit of (i), \citet{bonnabel2013stochastic} and \citet{zhang2016riemannian} develop scalable and stable optimization methods on Riemann manifold to enhance inference for M-VIs, and \citet{brubaker2012family}, \citet{byrne2013geodesic} and \citet{liu2016stochastic} develop efficient and scalable MCMCs on Riemann manifold.
In the spirit of (ii), \citet{hoffman2013stochastic} use natural gradient for M-VI, \citet{girolami2011riemann} and \citet{ma2015complete} develop efficient MCMCs with Riemann structure, and \citet{li2016preconditioned} apply the Riemannian MCMCs to the inference of Bayesian neural network.
Little work has been done to enhance P-VIs with Riemann structure.
\citet{gemici2016normalizing} attempt to generalize NF to Riemann manifold, but their method cannot be used for manifolds with no global c.s., such as hypersphere.
Although the method can be implemented in the almost-global c.s., unbounded distortion near the boundary of the c.s. would unavoidably occur, which would cause numerical instability.

In this work, we develop Riemannian Stein Variational Gradient Descent (RSVGD), the extension of SVGD to Riemann manifold.
Our method can be applied for both approximating the posterior on Riemann manifold (i), and efficient inference by exploring the Riemann structure of distributions (ii).
RSVGD inherits the significant advantages of SVGD, bringing the benefits to the field of inference on Riemann manifold, such as particle-efficiency and zero variational assumption.
Technically, it is highly non-trivial to extent the idea to Riemann manifold, as many subtle properties of Riemann manifold must be carefully considered, which may lead to completely different treatment.
We first review SVGD as an evolution under a dynamics of a flow and generalize the deduction to Riemann manifold.
Then we solve for the optimal dynamics by a novel method, where the treatment of SVGD fails due to the intrinsically different properties of general Riemann manifold from Euclidean space.
The expression in the embedded space is also derived for application to manifolds with no global c.s. like hyperspheres, which does not require choosing a c.s. and introduce no numerical problems.
As side products, we also develop Riemannian Stein's identity and Riemannian kernelized Stein discrepancy, as an extension of the corresponding concepts.
Finally, we apply RSVGD to the troublesome inference task of SAM, with its unique advantages validated in experiments.


\section{Preliminaries}

\subsection{Riemann Manifolds}

We briefly introduce basic concepts of Riemann manifold.
For more details please refer to common textbooks e.g. \citet{do1992riemannian}; \citet{abraham2012manifolds}.

\subsubsection{Basics}
Denote $\mf$ an $m$-dimensional Riemann manifold.
By definition at every point $A\in\mf$ there exists a local coordinate system (c.s.) $(U,\Phi)$, where $U\subset\mf$ is open and contains $A$, and $\Phi: U\to\re^m$ a homeomorphism between $U$ and $\Phi(U)$.
Denote $\cont_A^{\infty}$ and $\cont^{\infty}(\mf)$ as the set of functions $\mf\to\re$ that are smooth around $A$ and all over $\mf$, respectively.
A tangent vector $v$ at $A$ is a linear functional $\cont_A^{\infty}\to\re$ that satisfies $v[(fg)(\cdot)]=f(A)v[g(\cdot)]+g(A)v[f(\cdot)], \forall f,g\in\cont_A^{\infty}$.
Intuitively $v[f]$ is the directional derivative of $f$ at $A$ along the direction of $v$.
All such $v$ forms an $m$-dimensional linear space $T_A\mf$, called the tangent space at $A$.
Its natural basis $\{\partial_i\}_{i=1}^m$ under $(U,\Phi)$ is defined as $\partial_i(A)[f]\defas \frac{\partial (f\circ \Phi^{-1})}{\partial x^i}(x^1,\dots,x^m)|_{\Phi(A)}$ (also denoted as $\partial_i f|_A$).
We can then express $v$ in component: $v=v^i\partial_i$, where we adopt Einstein's convention that duplicated subscript and superscript are summed out.
A vector field $X$ on $\mf$ specifies at every $A\in\mf$ a tangent vector $X(A)\in T_A\mf$ smoothly with respect to (w.r.t.) $A$.
Denote $\tg(\mf)$ as the set of all such $X$.
A Riemann structure is equipped to $\mf$ if $\forall A\in\mf$, $T_A\mf$ is endowed with an inner product $g_A(\cdot,\cdot)$ (and $g_A$ is smooth in $A$).
In $(U,\Phi)$, for $u=u^i\partial_i, v=v^j\partial_j$, $g_A(u,v)=g_{ij}(A) u^i v^j$, where $g_{ij}(A)=g_A(\partial_i, \partial_j)$.

An object called Riemann volume form $\mu_g$ can be used to define a measure on $\mf$ via integral: the measure of a compact region $U\subset\mf$ is defined as $\int_U \mu_g$.
Thus for any probability distribution absolutely continuous \wrt $\mu_g$, we can define its probability density function (p.d.f.) $p$ \wrt $\mu_g$: $\prob(U) = \int_U p\mu_g$.
In the sequel, we require distributions such that their \pdf are smooth functions on $\mf$, and we would say ``distribution with \pdf $p$'' as ``distribution $p$''. 

\subsubsection{Flows, Dynamics and Evolving Distributions}
These concepts constitute the fundamental idea of SVGD and our RSVGD.
The notion of flow arises from the following fact:
for a vector field $X$ and a fixed point $A\in\mf$, there exist a subset $U\subset\mf$ containing $A$, and a one-parameter transformation $F_{(\cdot)}(\cdot): (-\delta, \delta)\times U \to \mf$ where $\delta\in\re$, such that $F_0(\cdot)$ is the identity map on $U$, and for $B\in U$ and $t_0\in(-\delta,\delta)$, $\frac{\ud}{\ud t}f(F_t(B))|_{t=t_0} = X(F_{t_0}(B))[f], \forall f\in\cont_{F_{t_0}(B)}^{\infty}$ (\citep{do1992riemannian}, Page~28).
We call $F_{(\cdot)}(\cdot)$ the local flow of $X$ around $A$.
Under some condition, e.g. $X$ has compact support, there exists a local flow that is global (i.e. $U=\mf$), which is called the \textbf{flow} of $X$.

We refer a \textbf{dynamics} here as a rule governing the motion on $\mf$ over time $t$.
Specifically, a dynamics gives the position $A(t)$ at any future time for any given initial position $A(0)$.
If $t\in\re$ and $A(t)$ is smooth for all $A(0)\in\mf$, we call it a continuous-time dynamics.
Obviously, a flow can be used to define a continuous-time dynamics: $A(t) = F_t(A_0), A(0) = A_0$.
Due to the correspondence between a vector field and a flow, we can also define a dynamics by a vector field $X$, and denote it as $\frac{\ud A}{\ud t} = X$.
Thus we would say ``the dynamics defined by the flow of $X$'' as ``dynamics $X$'' in the following.

Let a random variable (r.v.) obeying some distribution $p$ move under dynamics $X$ from time $0$ to $t$, which acts as a transformation on the r.v.
Then the distribution of the transformed r.v. also evolves along time $t$, and we denote it as $p_t$, and call it an \textbf{evolving distribution} under dynamics $X$.
Suppose there is a set of particles $\{A^{(i)}\}_{i=1}^N$ that distributes as $p$.
Let each particle move individually under dynamics $X$ for time $t$, then the new set of particles distributes as $p_t$.

\subsection{Reynolds Transport Theorem}

Reynolds transport theorem helps us to relate an evolving distribution to the corresponding dynamics.
It is a generalization of the rule of differentiation under integral, and is the foundation of fluid mechanics.
Let $X\in\tg(\mf)$ and $F_{(\cdot)}(\cdot)$ be its flow.
For smooth $f_{(\cdot)}(\cdot): \re\times\mf \to \re$ and any open subset $U\subset\mf$, the theorem states that
\begin{align*}
  \frac{\ud}{\ud t}\int_{F_t(U)} f_t \mu_g = \int_{F_t(U)} \left( \frac{\partial f_t}{\partial t} + \div(f_t X) \right) \mu_g,
\end{align*}
where $\div: \tg(\mf) \to \cont^{\infty}(\mf)$ is the divergence of a vector field.
In any local c.s., $\div(X) = \partial_i(\sqrt{|G|}X^i)/\sqrt{|G|}$, where $G$, commonly called Riemann metric tensor, is the $m\times m$ matrix comprised of $g_{ij}$ in that c.s., and $|G|$ is its determinant.
More details can be found in e.g. (\citep{romano2007continuum}, Page~164); (\citep{frankel2011geometry}, Page~142); (\citep{abraham2012manifolds}, Page~469).

\subsection{Stein Variational Gradient Descent (SVGD)}

We review SVGD~\cite{liu2016stein} from a perspective inspiring our generalization work to Riemann manifold.
SVGD is a particle-based variational inference method.
It updates particles by applying an appropriate dynamics on them so that the distribution of the particles, an evolving distribution, approaches (in the KL-divergence sense) the target distribution.
Note that in this case the manifold is the most common Euclidean space $\mf = \re^m$.
Denote the evolving distribution of the particles under dynamics $X$ as $q_t$, and let $p$ be the target distribution.
The first key result of SVGD is
\begin{align}
  -\frac{\ud}{\ud t}\kl(q_t || p) = \expect_{q_t} [X\trs \nabla\log p + \nabla\trs X],
  \label{eqn:olddirderiv}
\end{align}
which measures the rate of $q_t$ to approach $p$.
As our desiderata is to maximize $-\kl(q_t||p)$, a desirable dynamics should maximize this approaching rate.
This is analogous to the process of finding gradient when we want to maximize a function $f$: $\nabla f(x) = \alpha_0 v_0, (\alpha_0, v_0) = \max_{\|v\|=1} f'_v(x)$, where $f'_v(x)$ is the directional derivative along direction $v$.
Similarly, we call $-\frac{\ud}{\ud t}\kl(q_t||p)$ the \textbf{directional derivative} along vector field $X$, and \cite{liu2016stein} claims that the direction and magnitude of \textbf{functional gradient} of $-\kl(q_t||p)$ coincides with the maximizer and maximum of the directional derivative, and the dynamics of the functional gradient (itself a vector field) is the desired one.
Simulating this dynamics and repeating the procedure updates the particles to be more and more representative for the target distribution.

The above result reveals the relation between the directional derivative and dynamics $X$.
To find the functional gradient, we get the next task of solving the maximization problem.
Note that in Euclidean case, the tangent space at any point is isometrically isomorphic to $\re^m$, so $X$ can be described as $m$ smooth functions on $\re^m$.
\citet{liu2016stein} take $X$ from the product space $\hilb_K^m$ of the reproducing kernel Hilbert space (RKHS) $\hilb_K$ of some smooth kernel $K$ on $\re^m$.
$\hilb_K$ is a Hilbert space of some smooth functions on $\re^m$.
Its key property is that for any $A\in\re^m$, $K(A,\cdot)\in\hilb_K$, and $\langle f(\cdot),K(A,\cdot) \rangle_{\hilb_K} = f(A), \forall f\in\hilb_K$.
In $\hilb_K^m$ the maximization problem can be solved in closed form, which gives the functional gradient:
\begin{align}
  X^*(\cdot) = \expect_{q_t(A)} [K(A,\cdot)\nabla\log p(A) + \nabla K(A,\cdot)].
  \label{eqn:oldfuncgrad}
\end{align}
Note that the sample distribution $q_t$ appears only in the form of expectation, which can be estimated by merely samples from $q_t$.
This releases the assumption on the explicit form of $q_t$ thus gives it great flexibility.
It is a benefit of using KL-divergence to measure the difference between two distributions.

\section{Riemannian SVGD}

We now present Riemannian Stein Variational Gradient Descent (RSVGD), which has
many non-trivial considerations beyond SVGD and requires novel treatments.
We first derive the Riemannian counterpart of the directional derivative, then conceive a novel technique to find the functional gradient, in which case the SVGD technique fails.
Finally we express RSVGD in the embedded space of the manifold and give an instance for hyperspheres, which is directly used for our application.

\subsection{Derivation of the Directional Derivative}

Now $\mf$ is a general Riemann manifold.
We first derive a useful tool for deriving the directional derivative.
\begin{lem}[Continuity Equation on Riemann Manifold]
  \label{thm:cont}
  Let $p_t$ be the evolving distribution under dynamics $X$, where $X\in\tg(\mf)$. Then
  \begin{align}
	\frac{\partial p_t}{\partial t} = -\div(p_t X) = -X[p_t] - p_t \div(X). 
	\label{eqn:cont}
  \end{align}
\end{lem}
See Appendix\footnote{Appendix available at \url{http://ml.cs.tsinghua.edu.cn/~changliu/rsvgd/Liu-Zhu-appendix.pdf}} A1 for derivation, which involves Reynolds transport theorem.

We now present our first key result, the directional derivative of $-\kl(q_t||p)$\footnote{
  Appendix A2 presents the well-definedness of the KL-divergence on Riemann manifold $\mf$.
} \wrt vector field $X$.
\begin{thm}[Directional Derivative]
  \label{thm:dirderiv}
  Let $q_t$ be the evolving distribution under dynamics $X$, and $p$ a fixed distribution. Then the directional derivative is
  \begin{align*}
	-\frac{\ud}{\ud t} \kl(q_t||p) \!=\! \expect_{q_t} [\div(p X) / p] \!=\! \expect_{q_t} \!\big[\! X[\log p] \!+\! \div(\!X\!) \big].
  \end{align*}
\end{thm}
See Appendix A3 for proof.
Theorem~\ref{thm:dirderiv} is an extension of the key result of SVGD Eqn.~(\ref{eqn:olddirderiv}) to general Riemann manifold case.
In SVGD terms, we call $\stein_p X = X[\log p] + \div(X)$ the generalized Stein's operator.

Appendix A4 further discusses in detail the corresponding condition for Stein's identity to hold in this case.
Stein's identity refers to the equality $-\left.\frac{\ud}{\ud t}\kl(q_t||p)\right|_{t=t_0} = 0$ for any $t_0$ such that $q_{t_0}=p$.
Stein class is the set of $X$ that makes Stein's identity hold.
Stein's identity indicates that when $q_t$ reaches its optimal configuration $p$, the directional derivative along any direction $X$ in the Stein class is zero, in analogy to the well-known zero gradient condition in optimization.
These concepts play an important rule when using the functional gradient to measure the distance between two distributions, i.e. the Stein discrepancy \cite{liu2016kernelized}.

\subsection{Derivation of the Functional Gradient}
Now that we have the directional derivative expressed in terms of $X$, we get the maximization problem to find the functional gradient:
\begin{align}
  \max_{X\in\subtg, \|X\|_{\subtg}=1} \obj(X) \defas \expect_{q} \big[ X[\log p] + \div(X) \big],
  \label{eqn:obj}
\end{align}
where we omit the subscript of $q_t$ since it is fixed when optimizing \wrt $X$.
Ideally $\subtg$ should be $\tg(\mf)$, but for a tractable solution we may restrict $\subtg$ to some subset of $\tg(\mf)$ which is at least a normed space.
Once we get the maximizer of $\obj(X)$, denoted by $X^o$, we have the functional gradient $X^* = \obj(X^o)X^o \in \tg(\mf)$.

\subsubsection{Requirements}
Before we choose an appropriate $\subtg$, we first list three requirements on a reasonable $X^*$ (or equivalently on $X^o$ since they differ only in scale).
The first two requirements arise from special properties of general Riemann manifolds, which are so different from Euclidean spaces that make SVGD technique fail.

\begin{itemize}
  \item R1: $X^*$ is a valid vector field on $\mf$;
  \item R2: $X^*$ is coordinate invariant;
  \item R3: $X^*$ can be expressed in closed form, where $q$ appears only in terms of expectation.
\end{itemize}

We require R1 since all the above deductions are based on vector fields.
In the Euclidean case, any set of $m$ smooth functions $(f^1,\dots,f^m)$ satisfies R1.
But in general Riemann manifold, it is not enough that all the $m$ components of a vector field is smooth in any coordinate system (c.s.).
For example, due to the hairy ball theorem (\citep{abraham2012manifolds}, Theorem~8.5.13), vector fields on an even-dimensional hypersphere must have one zero-vector-valued point (critical point), which goes beyond the above condition.
This disables the idea to use the SVGD technique in the coordinate space of a manifold.

R2 is required to avoid ambiguity or arbitrariness of the solution.
Coordinate invariance is a key concept in the area of differential manifold.
By definition the most basic way to access a general manifold is via c.s., so we can define an object on the manifold by its expression in c.s.
If the form of the expression in any c.s. is the same, or equivalently the expression in a certain form refers to the same object in any c.s., we say that the form of the object is coordinate invariant.
For intersecting c.s. $(U,\Phi)$ and $(V,\Psi)$, if the expression in each c.s. with a same form gives different results, then the object is ambiguous on $U\cap V$.
One may argue that it is possible to use one certain c.s. to uniquely define the object, whatever the form in that c.s., but firstly for manifolds without global c.s., e.g. hyperspheres, we cannot define an object globally on the manifold in this manner, and secondly for manifolds that have global c.s., the choice of the certain c.s. may be arbitrary, since all the c.s. are equivalent to each other: there is no objective reason for choosing one specific c.s. other than other c.s.

Expression $\grad f = g^{ij} \partial_i f \partial_j$ is coordinate invariant, while $\nabla f = \sum_i \partial_i f \partial_i$ is not.
Suppose the above two expressions are written in c.s. $(U,\{x^i\}_{i=1}^m)$.
Let $(V,\{y^a\}_{a=1}^m)$ be another c.s. with $\tilde\partial_a$ and $\tilde g^{ab}$ the corresponding objects on it.
On $U\cap V$, $\grad f = \tilde g^{ab} \tilde\partial_a f \tilde\partial_b$, but $\nabla f = (\sum_i \frac{\partial y^a}{\partial x^i}\frac{\partial y^b}{\partial x^i}) \tilde\partial_a f \tilde\partial_b$ while we expect $\sum_a \tilde\partial_a f \tilde\partial_a$, which do not match for general c.s.
Consequently, the functional gradient of SVGD Eqn.~(\ref{eqn:oldfuncgrad}) is not coordinate invariant.
Again, the SVGD technique does not meet our demand here.

R3 cuts off the computational burden for optimizing \wrt $X$, and releases the assumption on the form of $q$ so as to adopt the benefit of SVGD of flexibility.

Before presenting our solution, we would emphasize the difficulty by listing some possible attempts that actually fail.
First note that $\subtg$ is a subspace of $\tg(\mf)$, which is a linear space on both $\re$ and $\cont^{\infty}(\mf)$.
But in either case $\tg(\mf)$ is infinite dimensional and it is intractable to express vector fields in it.
Secondly, as the treatment of SVGD, one may consider expressing vector fields component-wise.
This idea specifies a tangent vector at each point, thus would easily violate R1 and R2, due to the intrinsically different characteristic of general Riemann manifolds from Euclidean space: tangent spaces at different points are not the same.
R1 and R2 focus on global properties, one has to link tangent spaces at different points.
Taking a c.s. could uniformly express the tangent vectors on the c.s., but this is not enough as stated above.
One may also consider transforming tangent vectors at different points to one certain pivot point by parallel transport, but the arbitrariness of the pivot point would dissatisfy R2.
The third possibility is to view a vector field as a map from $\mf$ to $T\mf$, the tangent bundle of $\mf$, but $T\mf$ is generally only a manifold but not a linear space, so we cannot apply theories of learning vector-valued functions (e.g. \citet{micchelli2005learning}).

\subsubsection{Solution}

We first present our solution, then check the above requirements.
Note that theories of kernel and reproducing kernel Hilbert space (RKHS) (see e.g. \citet{steinwart2008support}, Chapter~4) also apply to manifold case.
Let $K:\mf\times\mf\to\re$ be a smooth kernel on $\mf$ and $\hilb_K$ its RKHS.
We require $K$ such that zero function is the only constant function in $\hilb_K$.
Such kernels include the commonly used Gaussian kernel (\citep{steinwart2008support}, Corollary~4.44).

We choose $\subtg=\{\grad f|f\in\hilb_K\}$, where $\mathrm{grad}: \cont^{\infty}(\mf)\to\tg(\mf)$ is the gradient of a smooth function, which is a valid vector field.
In any c.s., $\grad f = g^{ij}\partial_i f \partial_j$, where $g^{ij}$ is the $(i,j)$-th entry of $G^{-1}$.
The following result indicates that $\subtg$ can be made into an inner product space.
\begin{lem}
  \label{thm:subtg}
  With a proper inner product, $\subtg$ is isometrically isomorphic to $\hilb_K$, thus a Hilbert space.
\end{lem}
\begin{proof}
  Define $\iota:\hilb_K\to\subtg, f\mapsto\grad f$, which is linear: $\forall\alpha\in\re,f,h\in\hilb_K,\iota(\alpha f+h) = \alpha\grad f + \grad h = \alpha\iota(f) + \iota(h)$.
  For any $f,h\in\hilb_K$ that satisfy $\iota(f)=\iota(h)$, we have $\grad(f-h)=0$, $f-h={\const}_{\hilb_K} = 0_{\hilb_K}$, where the last equality holds for our requirement on $K$.
  So $f=h$ thus $\iota$ is injective.
  By definition of $\subtg$, $\iota$ is surjective, so it is an isomorphism between $\subtg$ and $\hilb_K$.

  Define $\langle\cdot,\cdot\rangle_{\subtg}: \langle X,Y \rangle_{\subtg} = \langle \iota^{-1}(X),\iota^{-1}(Y) \rangle_{\hilb_K}$, $\forall X,Y\in\subtg$.
  By the linearity of $\iota$ and that $\langle \cdot,\cdot \rangle_{\hilb_K}$ is an inner product, one can easily verify that $\langle\cdot,\cdot\rangle_{\subtg}$ is an inner product on $\subtg$, and that $\iota$ is an isometric isomorphism.
\end{proof}

Next we present our second key result, that the objective $\obj(X)$ can be cast as an inner product in $\subtg$.
\begin{thm}
  \label{thm:optsolution}
  For $(\subtg,\langle \cdot,\cdot \rangle_{\subtg})$ defined above and $\obj$ the objective in Eqn.~(\ref{eqn:obj}), we have $\obj(X) = \langle X,\hat X \rangle_{\subtg}$, where
  \begin{gather}
	\hat X=\grad\hat f,\nonumber\\
	\hat f(\cdot) = \expect_{q(A)}\Big[\big(\grad K\!(A,\cdot)\big)[\log p(A)] + \Delta K\!(A,\cdot)\Big],
	\label{eqn:optsolution}
  \end{gather}
  and $\Delta f \defas \div(\grad f)$ is the Beltrami-Laplace operator. 
\end{thm}
\begin{proof}[Proof (sketch)]
  Use $f=\iota^{-1}(X)\in\hilb_K$ to express the vector field $X=\grad f$.
  Expand $\obj$ in some c.s., then cast partial derivatives into the form of inner product in $\hilb_K$ based on the results of \citet{zhou2008derivative}.
  Rearranging terms by the linearity of inner product and expectation, $\obj$ can be written as an inner product in $\hilb_K$, thus an inner product in $\subtg$ due to their isomorphism.
  See Appendix A5 for details.
\end{proof}

In the form of an inner product, the maximization problem~(\ref{eqn:obj}) solves as: $X^o = \hat X/\|\hat X\|_{\subtg}$.
Then the functional gradient is $X^* = \obj(X^o)X^o = \frac{\langle \hat X,\hat X \rangle_{\subtg}}{\|\hat X\|_{\subtg}} \cdot \frac{\hat X}{\|\hat X\|_{\subtg}} = \hat X$.
\begin{cor}(Functional Gradient)
  \label{thm:funcgrad}
  For $(\subtg,\langle \cdot,\cdot \rangle_{\subtg})$ taken as above, the functional gradient is $X^* = \hat X$.
\end{cor}

Now we check our solution for the three requirements.
Since $\grad$, $\div$, $\Delta$, expectation, and the action of tangent vector on smooth function $v[f]$ are all coordinate invariant objects, $\hat f$ is coordinate invariant thus a valid smooth function on $\mf$.
So its gradient $\hat X$ is a valid vector field (R1) and also coordinate invariant (R2).
R3 is obvious from Eqn.~(\ref{eqn:optsolution}).

As a bonus, we present the optimal value of the objective:
\begin{align*}
  \obj(\hat X) =& \expect_{q}\expect_{q'}\Big[ \big(\gradn'\log p'\big)\big[ (\grad\log p)[K] \big] + \Delta'\Delta K \\
  & + (\gradn'\log p')[\Delta K] + (\grad\log p)[\Delta'K] \Big],
\end{align*}
where $K=K(A,A')$, and notations with prime ``$\,'\,$'' take $A'$ as argument and others take $A$.
We call it Riemannian Kernelized Stein Discrepancy (RKSD) between distributions $q$ and $p$, a generalization of Kernelized Stein Discrepancy (\citenop{liu2016kernelized}; \citenop{chwialkowski2016kernel}).

After deriving the functional gradient $X^* = \hat X$, we can simulate its dynamics in any c.s. (denoted as $(U,\{x^i\}_{i=1}^m)$) by $x^i(t+\varepsilon) = x^i(t) + \varepsilon \hat X^i(x(t))$ for each component $i$, which is a 1st-order approximation of the flow of $\hat X$, where
\begin{align}
  \hat X^i(A') =  g'^{ij} \partial'_j \expect_q \Big[& \big( g^{ab}\partial_a\log(p\sqrt{|G|}) + \partial_a g^{ab} \big) \partial_b K \nonumber\\
  & + g^{ab}\partial_a\partial_b K \Big].
  \label{eqn:csvec}
\end{align}
This is the update rule of particles for inference tasks on Euclidean space, where the expectation is estimated by averaging over current particles, and the Riemann metric tensor $g_{ij}$ is taken as the subtraction of the Fisher information of likelihood with the Hessian of prior p.d.f., as adopted by \citet{girolami2011riemann}.
Note that $g^{ij}$ is the entry of the inverse matrix.

\subsection{Expression in the Embedded Space}

From the above discussion, we can simulate the optimal dynamics in c.s.
But it is not always the most convenient approach.
For some manifolds, like hyperspheres $\sph^{n-1}\defas\{x\in\re^n | \|x\|_2 = 1\}$ or Stiefel manifold $\stiefel_{m,n}\defas\{M\in\re^{m\times n} | M\trs M = I_m \}$ \cite{james1976topology}, on one hand, they have no global c.s. so we have to change c.s. constantly while simulating, and we have to compute $g^{ij}$, $|G|$ and $\partial_i K$ in each c.s.
It is even hard to find a c.s. of a Stiefel manifold.
On the other hand, such manifolds are defined as a subset of some Euclidean space, 
which is a natural embedded space.
This motivates us to express the dynamics of the functional gradient in the embedded space and simulate in it.

Formally, an embedding of a manifold $\mf$ is a smooth injection $\Xi:\mf\to\re^n$ for some $n\ge m$.
For a Riemann manifold, $\Xi$ is isometric if $g_{ij} = \sum_{\alpha=1}^n \frac{\partial y^{\alpha}}{\partial x^i}\frac{\partial y^{\alpha}}{\partial x^j}$ in any c.s. $(U,\Phi)$, where $\frac{\partial y^{\alpha}}{\partial x^i}$ is for $y=\xi(x)$ with $\xi \defas \Xi\circ\Phi^{-1}$.
The Hausdorff measure on $\Xi(\mf)\subset\re^n$ induces a measure on $\mf$, \wrt which we have \pdf $p_H$.
We recognize that for isometric embedding, $p=p_H$.

Since $\Xi$ is injective, $\Xi^{-1}: \Xi(\mf) \to \mf$ can be well-defined, so can $\xi^{-1}$ for any c.s.
The following proposition gives a general expression of the functional gradient in an isometrically embedded space.
\begin{prop}
  \label{thm:genemb}
  Let all the symbols take argument in the isometrically embedded space $\re^n$ (with orthonormal basis $\{y^{\alpha}\}_{\alpha=1}^n)$) via composed with $\Xi^{-1}$ or $\xi^{-1}$.
  We have $\hat X' = (I_n - N'{N'}\trs)\nabla' \hat f'$,
  \begin{align}
	\hat f' = \expect_{q} \Big[& \Big(\nabla\log \big(p\sqrt{|G|}\big)\Big)\trs \Big(I_n - NN\trs\Big) (\nabla K) \nonumber\\
	& + \nabla\trs\nabla K - \tr\Big( N\trs(\nabla\nabla\trs K)N \Big) \nonumber\\
	& + \Big( (M\trs\nabla)\trs (G^{-1} M\trs) \Big) (\nabla K) \Big],
	\label{eqn:genemb}
  \end{align}
  where $I_n\in\re^{n\times n}$ is the identity matrix, $\nabla = (\partial_{y^1}, \dots, \partial_{y^n})\trs$, $M\in\re^{n\times m}: M_{\alpha i} = \frac{\partial y^{\alpha}}{\partial x^i}$, $N(A)\in\re^{n\times(n-m)}$ is the set of orthonormal basis of the orthogonal complement of $\Xi_*(T_A\mf)$ (the $\Xi$-pushed-forward tangent space, an $m$-dimensional linear subspace of $\re^n$)
, and $\tr(\cdot)$ is the trace of a matrix.
\end{prop}
See Appendix A6.1 for derivation.
Note that $N$ does not depend on the choice of c.s. of $\mf$ while $M$ and $G$ do, but the final result does not.
Simulating the dynamics is quite different from the coordinate space case due to the constraint of $\Xi(\mf)$.
A 1st-order approximation of the flow is
\begin{align*}
  y(t+\varepsilon) = \Exp_{y(t)}\big(\varepsilon \hat X(y(t)) \big),
\end{align*}
where $\Exp_A$ is the exponential map at $A\in\mf$, which maps $v\in T_A\mf$ to the end point of moving $A$ along the geodesic determined by $v$ for unit time (parameter of the geodesic).
In $\re^n$, $A$ is moved along the straight line in the direction of $v$ for length $\|v\|$, and in $\sph^{n-1}$, $A$ is moved along the great circle (orthodrome) tangent to $v$ at $A$ for length $\|v\|$.

\subsubsection{Instantiation for Hyperspheres}
We demonstrate the above result with the instance of $\mf=\sph^{n-1}$.
\begin{prop}
  \label{thm:sphemb}
  For $\sph^{n-1}$ isometrically embedded in $\re^n$ with orthonormal basis $\{y^{\alpha}\}_{\alpha=1}^n$, we have $\hat X' = (I_n - y'{y'}\trs)\nabla'\hat f'$,
  \begin{align}
	\hat f' = \expect_q \Big[& (\nabla\log p\big)\trs (\nabla K) + \nabla\trs\nabla K - y\trs\big( \nabla\nabla\trs K \big) y \nonumber\\
	& - (y\trs \nabla\log p + n - 1)y\trs \nabla K \Big].
	\label{eqn:sphemb}
  \end{align}
\end{prop}
\noindent Note that the form of the expression does not depend on any c.s. of $\sph^{n-1}$.
The exponential map on $\sph^{n-1}$ is given by
\begin{align*}
  \Exp_y(v) = y\cos(\|v\|) + (v/\|v\|) \sin(\|v\|).
\end{align*}
See Appendix A6.2 for more details.
Appendix A6.3 further provides the expression for the product manifold of hyperspheres $(\sph^{n-1})^P$, which is the specific manifold on which the inference task of Spherical Admixture Model is defined.

\section{Experiments}
We find the na\"{i}ve mini-batch implementation of RSVGD does not perform well in experiments, which may require further investigation on the impact of gradient noise on the dynamics.
So we only focus on the full-batch performance of RSVGD in experiments\footnote{
  Codes and data available at \url{http://ml.cs.tsinghua.edu.cn/~changliu/rsvgd/}
}.

\subsection{Bayesian Logistic Regression}

We test the empirical performance of RSVGD on Euclidean space (Eqn.~\ref{eqn:csvec}) by the inference task of Bayesian logistic regression (BLR).
BLR generates latent variable from prior $w\sim\normal(0,\alpha I_m)$, and for each datum $x_d$, draws its label from Bernoulli distribution $y_d\sim \Bern(s(w\trs x_d))$, where $s(x)=1/(1+e^{-x})$.
The inference task is to estimate the posterior $p(w|\{x_d\}, \{y_d\})$.
Note that for this model, as is shown in Appendix A7, all the quantities involving Riemann metric tensor $G$ has a closed form, except $G^{-1}$, which can be solved by either direct numerical inversion, or an iterative method over the dataset.

\noindent\textbf{Kernel}
As stated before, in the Euclidean case we essentially conduct RSVGD on the distribution manifold.
To specify a kernel on it, we note that the Euclidean space of the latent variable is a global c.s. $(\Phi, \re^m)$ of the manifold.
By the mapping $\Phi$, any kernel on $\re^m$ is a kernel on the manifold (\citet{steinwart2008support}, Lemma~4.3).
In this sense we choose the Gaussian kernel on $\re^m$. 
Furthermore, we implement the kernel for both SVGD and RSVGD as the summation of several Gaussian kernels with different bandwidths, which is still a valid kernel (\citet{steinwart2008support}, Lemma~4.5).
We find it slightly better than the median trick of SVGD in our experiments.

\noindent\textbf{Setups}
We compare RSVGD with SVGD (full-batch) for test accuracy along iteration. 
We fix $\alpha=0.01$ and use 100 particles for both methods.
We use the Splice19 dataset (1,000 training entries, 60 features), one of the benchmark datasets compiled by \citet{mika1999fisher}, and the Covertype dataset (581,012 entries, 54 features) also used by \citet{liu2016stein}.
RSVGD updates particles by the aforementioned 1st-order flow approximation, which is effectively the vanilla gradient descent, while SVGD uses the recommended AdaGrad with momentum.
We also tried gradient descent for SVGD, with no better outcomes.

\begin{figure}
  \centering
    \subfigure[On Splice19 dataset]{\includegraphics[width=0.22\textwidth]{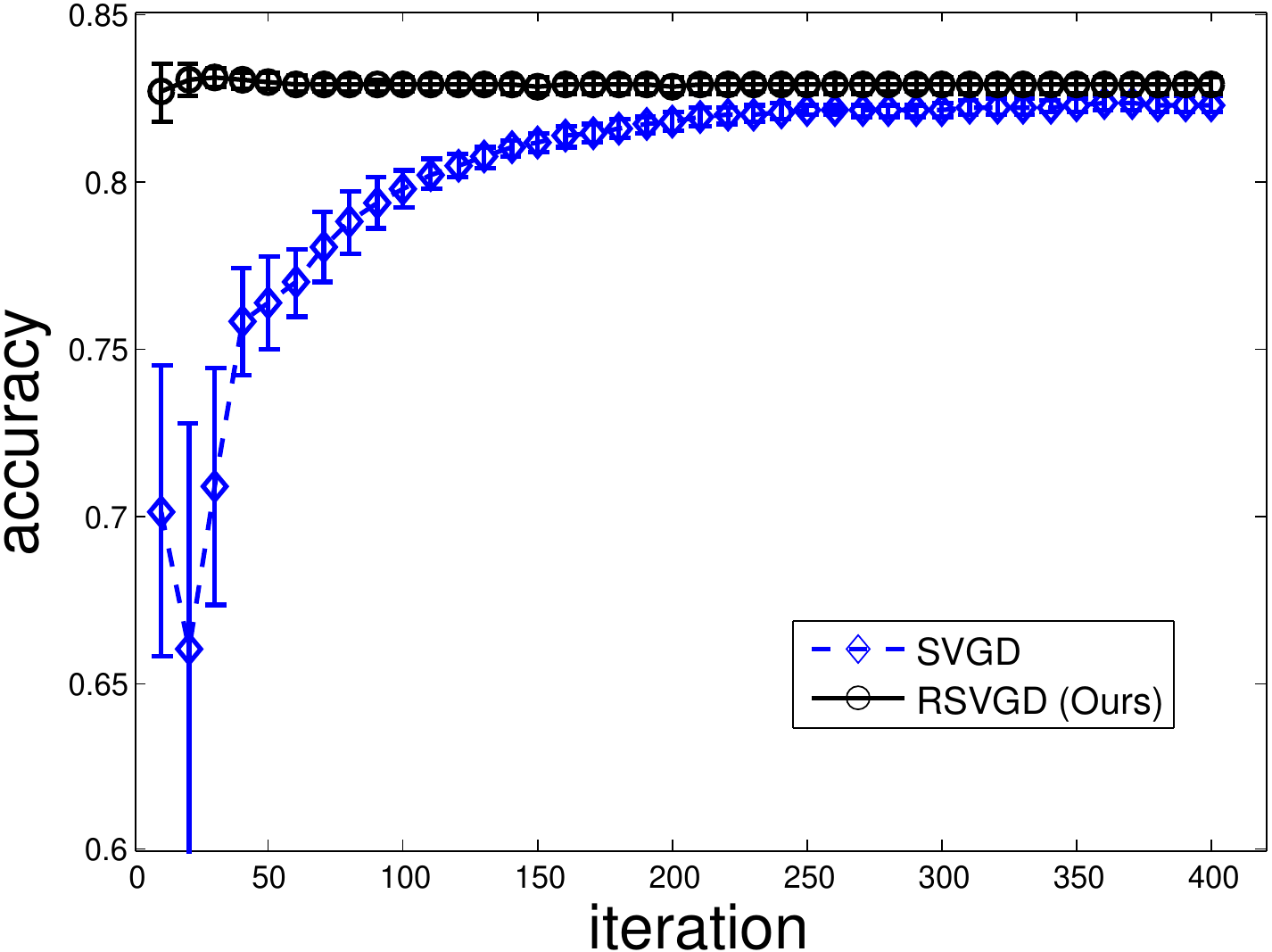} \label{fig:blriter-splice}}
    \subfigure[On Covertype dataset]{\includegraphics[width=0.22\textwidth]{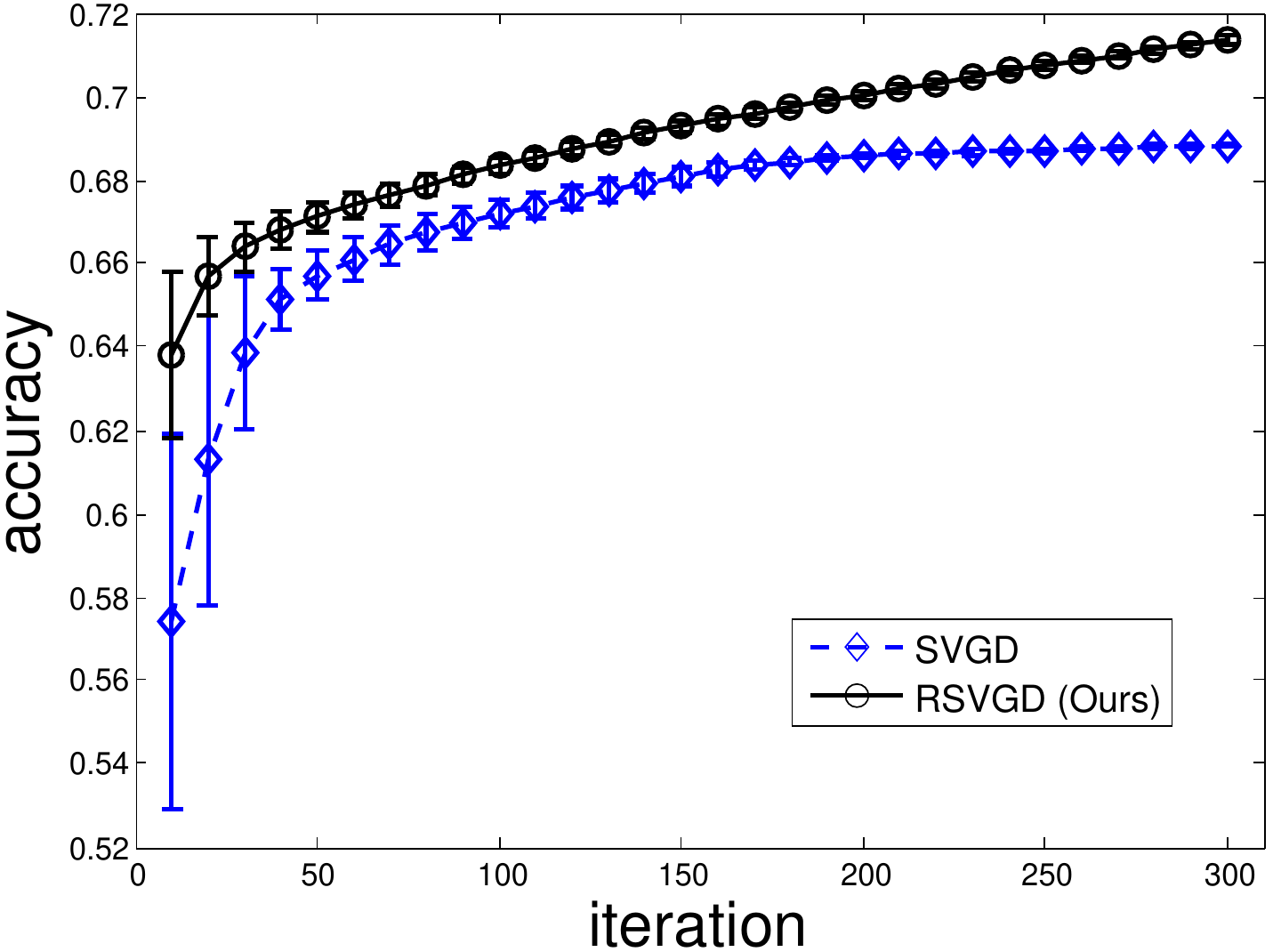} \label{fig:blriter-covtype}}
	\caption{Test accuracy along iteration for BLR.
	  Both methods are run 20 times on Splice19 and 10 times on Covertype.
	  Each run on Covertype uses a random train(80\%)-test(20\%) split as in~\cite{liu2016stein}.
	}
	\label{fig:blriter}
\end{figure}

\noindent\textbf{Results}
We see from Fig. \ref{fig:blriter} that RSVGD makes more effective updates than SVGD on both datasets, indicating the benefit of RSVGD to utilize the distribution geometry for Bayesian inference.
Although AdaGrad with momentum, which SVGD uses, counts for a method for estimating the geometry empirically \cite{duchi2011adaptive}, our method provides a more principled solution, with more precise results.

\subsection{Spherical Admixture Model}

We investigate the advantages of RSVGD on Riemann manifold (Eqn.~\ref{eqn:genemb}) by the inference task of Spherical Admixture Model (SAM) \cite{reisinger2010spherical}, which is a topic model for data on hyperspheres, e.g. \textit{tf-idf} feature of documents.
The model first generates corpus mean $\mu\sim\vmf(m,\kappa_0)$ and topics $\{\beta_k\}_{k=1}^P: \beta_k \sim\vmf(\mu,\sigma)$, then for document $d$, generates its topic proportion $\theta_d \sim\dir(\alpha)$ and content $v_d\sim\vmf(\beta\theta_d/\|\beta\theta_d\|, \kappa)$, where $\vmf$ is the von Mises-Fisher distribution \cite{mardia2000distributions} for random variable on hyperspheres, and $\dir$ is the Dirichlet distribution.

The inference task of SAM is to estimate the posterior of the topics $p(\beta|v)$.
Note that the topics $\beta=(\beta_1,\dots,\beta_P)$ lies in the product manifold of hyperspheres $(\sph^{n-1})^P$.

\noindent\textbf{Kernel}
Like the Gaussian kernel in the Euclidean case, we use the vMF kernel $K(y,y') = \exp(\kappa y\trs y')$ on hyperspheres.
Note that the vMF kernel on $\sph^{n-1}$ is the restriction of the Gaussian kernel on $\re^n$: $\exp(-\frac{\kappa}{2}\|y-y'\|^2) = \exp(-\kappa)\exp(\kappa y\trs\! y')$ for $y,y'\in\sph^{n-1}$, so it is a valid kernel on $\sph^{n-1}$ (\citet{steinwart2008support}, Lemma~4.3).
We also recognized that the arcsine kernel $K(y,y') = \arcsin(y\trs y')$ is also a kernel on $\sph^{n-1}$, due to the non-negative Taylor coefficients of the arcsine function (\citet{steinwart2008support}, Lemma~4.8).
But unfortunately it does not work well in experiments.

For a kernel on $(\sph^{n-1})^P$, we set $K(y,y') = \prod_{k=1}^P \exp( \kappa y_{(k)}\trs y_{(k)}' )$.
Again we use the summed kernel trick for RSVGD.

\noindent\textbf{Setups}
The manifold constraint isolates the task from most prevalent inference methods, including SVGD.
A mean-field variational inference method (VI) is proposed by the original work of SAM.
\citet{liu2016stochastic} present more methods based on advanced Markov chain Monte Carlo methods (MCMCs) on manifold, including Geodesic Monte Carlo (GMC) \cite{byrne2013geodesic}, and their proposed Stochastic Gradient Geodesic Monte Carlo (SGGMC), a scalable mini-batch method. 
We implement the two MCMCs in both the standard sequential (-seq) way, and the parallel (-par) way: run multiple chains and collect the last sample on each chain.
To apply RSVGD for inference, as with GMC and SGGMC cases, we adopt the framework of \citet{liu2016stochastic}, which directly provides an estimate of the all-we-need information $\nabla_{\beta}\log p(\beta|v)$. 
We run all the methods on the 20News-different dataset (1,666 training entries, 5,000 features) with default hyperparameters as the same as \cite{liu2016stochastic}.
We use epoch (the amount of visit to the dataset) instead of iteration since SGGMCb is run with mini-batch.

\begin{figure}
  \centering
    \subfigure[Results with 100 particles]{\includegraphics[width=0.22\textwidth]{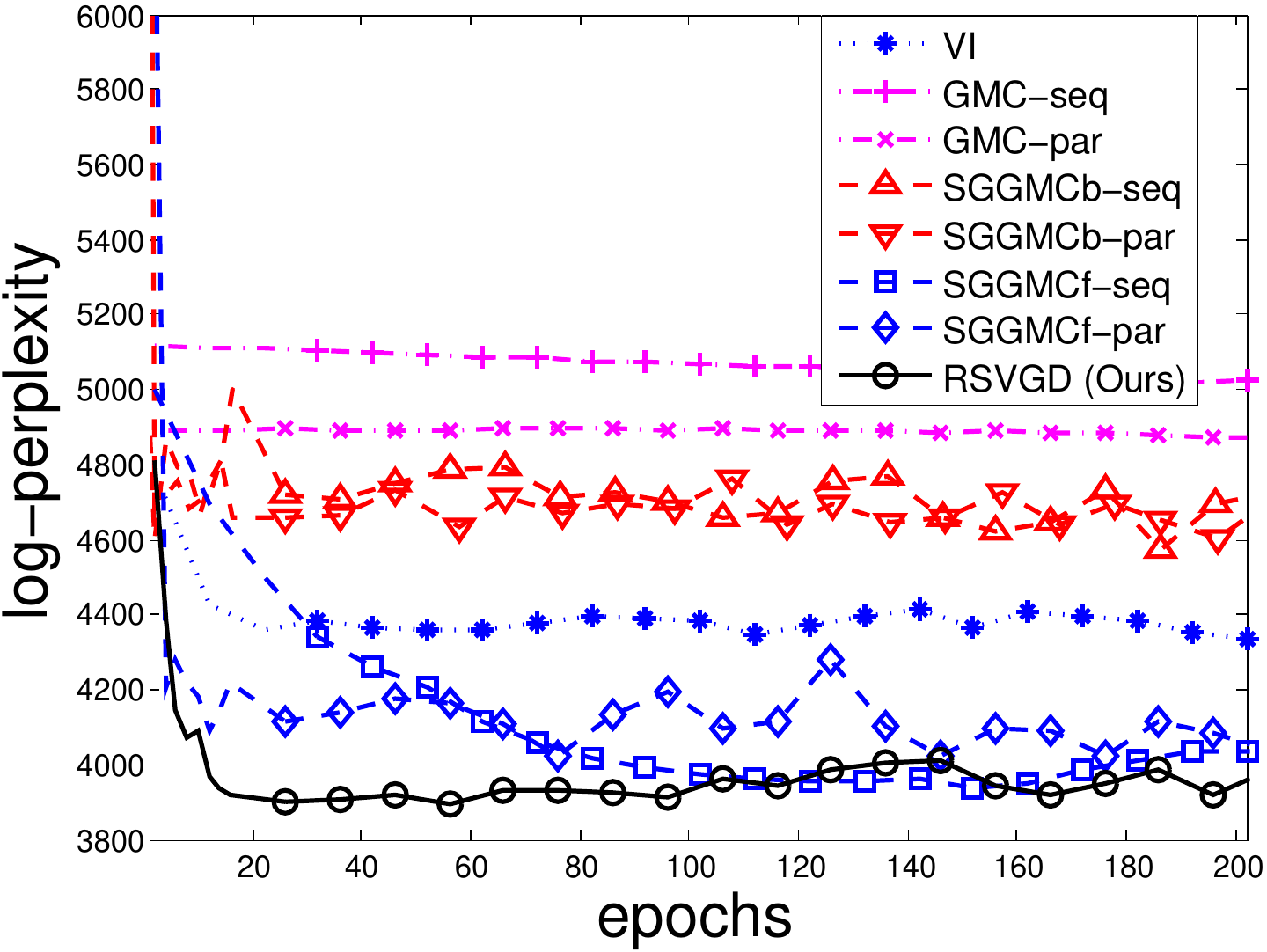} \label{fig:samperp-epoch}}
    \subfigure[Results at 200 epochs]{\includegraphics[width=0.22\textwidth]{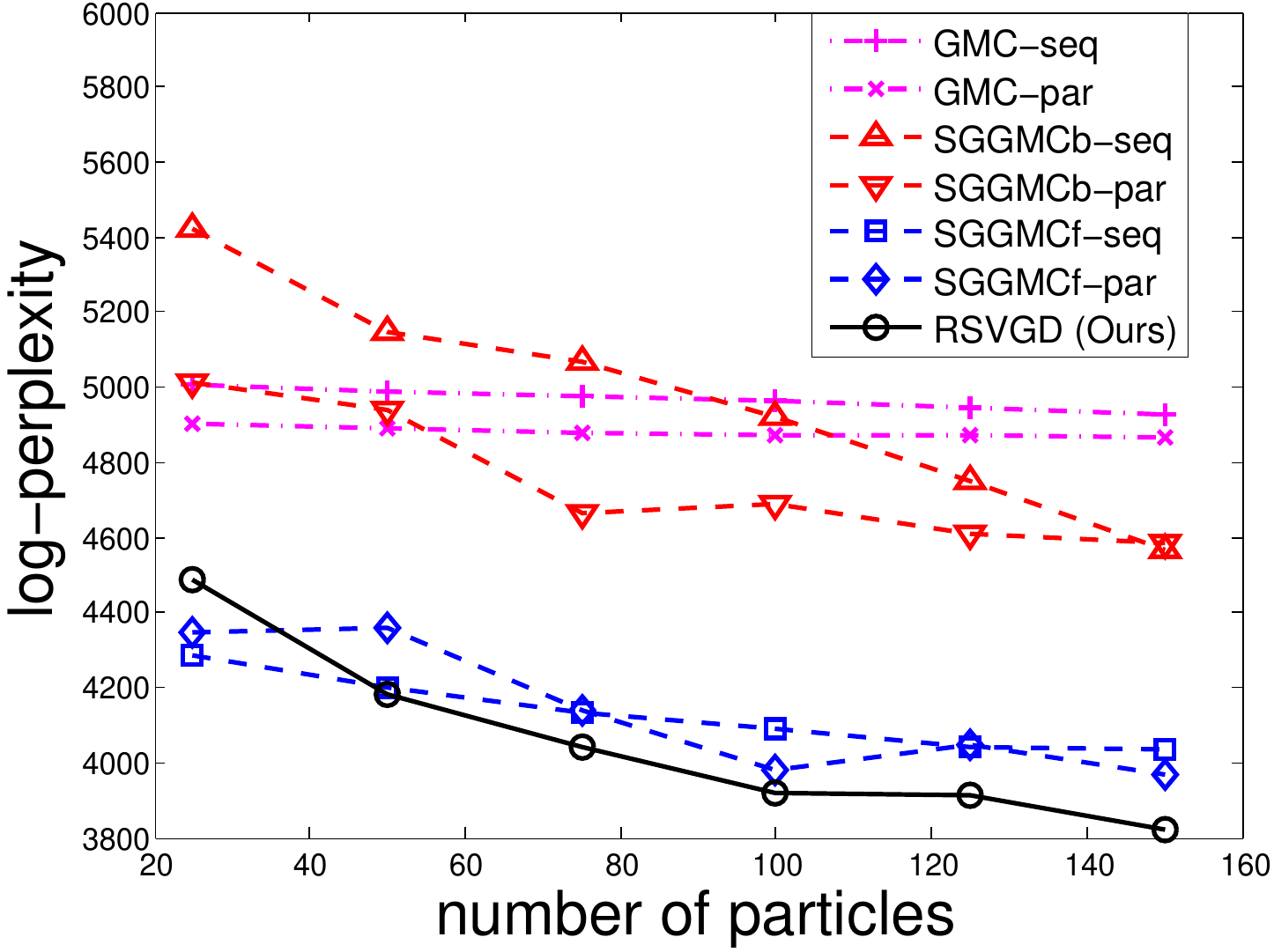} \label{fig:samperp-nump}}
	\caption{Results on the SAM inference task on 20News-different dataset, in log-perplexity.
	  We run SGGMCf for full batch and SGGMCb for a mini-batch size of 50.
	}
  \label{fig:samperp}
\end{figure}

\noindent\textbf{Results}
Fig.~\ref{fig:samperp-epoch} shows that RSVGD makes the most effective progress along epoch, indicating its iteration-effectiveness.
VI converges fast but to a less satisfying state due to its restrictive variational assumption, while RSVGD, by the advantage of high approximation flexibility, achieves a better result comparable to MCMC results.
GMC methods make consistent but limited progress, where 200 epochs is too short for them.
Although SGGMC methods perform outstandingly in \cite{liu2016stochastic}, they are embarrassed here by limited particle size. 
Both fed on full-batch gradient, SGGMCf has a more active dynamics than GMC thus can explore a broader region, but still not as efficient as RSVGD.
The -seq and -par versions of an MCMC behave similarly, although -par seems better at early stage since it has more particles.

Fig.~\ref{fig:samperp-nump} presents the results over various numbers of particles, where we find the particle-efficiency of RSVGD.
For fewer particles, SGGMCf methods have the chance to converge well in 200 epochs thus can be better than RSVGD.
For more particles MCMCs make less salient progress since the positive autocorrelation limits the representativeness of limited particles.

\section{Conclusion}

We develop Riemannian Stein Variational Gradient Descent (RSVGD), an extension of SVGD \cite{liu2016stein} to Riemann manifold.
We generalize the idea of SVGD and derive the directional derivative on Riemann manifold.
To solve for a valid and close-formed functional gradient, we first analyze the requirements by Riemann manifold and illustrate the failure of SVGD techniques in our case, then propose our solution and validate it.
Experiments show the benefit of utilizing distribution geometry on inference tasks on Euclidean space, and the advantages of particle-efficiency, iteration-effectiveness and approximation flexibility on Riemann manifold.

Possible future directions include exploiting Riemannian Kernelized Stein Discrepancy which would be more appropriate with a properly chosen manifold.
Mini-batch version of RSVGD for scalability is also an interesting direction, since the na\"{i}ve implementation does not work well as mentioned.
Applying RSVGD to a broader stage is also promising, including Euclidean space tasks like deep generative models and Bayesian neural networks (with Riemann metric tensor estimated in the way of \cite{li2016preconditioned}), and Riemann manifold tasks like Bayesian matrix completion \cite{song2016bayesian} on Stiefel manifold.

\subsubsection*{Acknowledgements}
This work is supported by the National NSF of China (Nos. 61620106010, 61621136008, 61332007) and Tiangong Institute for Intelligent Computing. 
We thank Jingwei Zhuo and Yong Ren for inspiring discussions.

\bibliographystyle{aaai}
\bibliography{Liu-Zhu}

\begin{thebibliography}{}

\bibitem[\protect\citeauthoryear{Abraham, Marsden, and
  Ratiu}{2012}]{abraham2012manifolds}
Abraham, R.; Marsden, J.~E.; and Ratiu, T.
\newblock 2012.
\newblock {\em Manifolds, tensor analysis, and applications}, volume~75.
\newblock Springer Science \& Business Media.

\bibitem[\protect\citeauthoryear{Amari and Nagaoka}{2007}]{amari2007methods}
Amari, S.-I., and Nagaoka, H.
\newblock 2007.
\newblock {\em Methods of information geometry}, volume 191.
\newblock American Mathematical Soc.

\bibitem[\protect\citeauthoryear{Amari}{2016}]{amari2016information}
Amari, S.-I.
\newblock 2016.
\newblock {\em Information geometry and its applications}.
\newblock Springer.

\bibitem[\protect\citeauthoryear{Bonnabel}{2013}]{bonnabel2013stochastic}
Bonnabel, S.
\newblock 2013.
\newblock Stochastic gradient descent on riemannian manifolds.
\newblock {\em IEEE Transactions on Automatic Control} 58(9):2217--2229.

\bibitem[\protect\citeauthoryear{Brubaker, Salzmann, and
  Urtasun}{2012}]{brubaker2012family}
Brubaker, M.~A.; Salzmann, M.; and Urtasun, R.
\newblock 2012.
\newblock A family of mcmc methods on implicitly defined manifolds.
\newblock In {\em Proceedings of the 15th International Conference on
  Artificial Intelligence and Statistics (AISTATS)},  161--172.

\bibitem[\protect\citeauthoryear{Byrne and Girolami}{2013}]{byrne2013geodesic}
Byrne, S., and Girolami, M.
\newblock 2013.
\newblock Geodesic monte carlo on embedded manifolds.
\newblock {\em Scandinavian Journal of Statistics} 40(4):825--845.

\bibitem[\protect\citeauthoryear{Chwialkowski, Strathmann, and
  Gretton}{2016}]{chwialkowski2016kernel}
Chwialkowski, K.; Strathmann, H.; and Gretton, A.
\newblock 2016.
\newblock A kernel test of goodness of fit.
\newblock In {\em International Conference on Machine Learning},  2606--2615.

\bibitem[\protect\citeauthoryear{Dai \bgroup et al\mbox.\egroup
  }{2016}]{dai2016provable}
Dai, B.; He, N.; Dai, H.; and Song, L.
\newblock 2016.
\newblock Provable bayesian inference via particle mirror descent.
\newblock In {\em Artificial Intelligence and Statistics},  985--994.

\bibitem[\protect\citeauthoryear{Do~Carmo}{1992}]{do1992riemannian}
Do~Carmo, M.~P.
\newblock 1992.
\newblock {\em Riemannian Geometry}.

\bibitem[\protect\citeauthoryear{Duchi, Hazan, and
  Singer}{2011}]{duchi2011adaptive}
Duchi, J.; Hazan, E.; and Singer, Y.
\newblock 2011.
\newblock Adaptive subgradient methods for online learning and stochastic
  optimization.
\newblock {\em Journal of Machine Learning Research} 12(Jul):2121--2159.

\bibitem[\protect\citeauthoryear{Frankel}{2011}]{frankel2011geometry}
Frankel, T.
\newblock 2011.
\newblock {\em The geometry of physics: an introduction}.
\newblock Cambridge University Press.

\bibitem[\protect\citeauthoryear{Gemici, Rezende, and
  Mohamed}{2016}]{gemici2016normalizing}
Gemici, M.~C.; Rezende, D.; and Mohamed, S.
\newblock 2016.
\newblock Normalizing flows on riemannian manifolds.
\newblock {\em arXiv preprint arXiv:1611.02304}.

\bibitem[\protect\citeauthoryear{Girolami and
  Calderhead}{2011}]{girolami2011riemann}
Girolami, M., and Calderhead, B.
\newblock 2011.
\newblock Riemann manifold langevin and hamiltonian monte carlo methods.
\newblock {\em Journal of the Royal Statistical Society: Series B (Statistical
  Methodology)} 73(2):123--214.

\bibitem[\protect\citeauthoryear{Haarnoja \bgroup et al\mbox.\egroup
  }{2017}]{haarnoja2017reinforcement}
Haarnoja, T.; Tang, H.; Abbeel, P.; and Levine, S.
\newblock 2017.
\newblock Reinforcement learning with deep energy-based policies.
\newblock {\em arXiv preprint arXiv:1702.08165}.

\bibitem[\protect\citeauthoryear{Hoffman \bgroup et al\mbox.\egroup
  }{2013}]{hoffman2013stochastic}
Hoffman, M.~D.; Blei, D.~M.; Wang, C.; and Paisley, J.
\newblock 2013.
\newblock Stochastic variational inference.
\newblock {\em The Journal of Machine Learning Research} 14(1):1303--1347.

\bibitem[\protect\citeauthoryear{James}{1976}]{james1976topology}
James, I.~M.
\newblock 1976.
\newblock {\em The topology of Stiefel manifolds}, volume~24.
\newblock Cambridge University Press.

\bibitem[\protect\citeauthoryear{Li \bgroup et al\mbox.\egroup
  }{2016}]{li2016preconditioned}
Li, C.; Chen, C.; Carlson, D.~E.; and Carin, L.
\newblock 2016.
\newblock Preconditioned stochastic gradient langevin dynamics for deep neural
  networks.
\newblock In {\em AAAI}, volume~2, ~4.

\bibitem[\protect\citeauthoryear{Liu and Wang}{2016}]{liu2016stein}
Liu, Q., and Wang, D.
\newblock 2016.
\newblock Stein variational gradient descent: A general purpose bayesian
  inference algorithm.
\newblock In {\em Advances in Neural Information Processing Systems},
  2370--2378.

\bibitem[\protect\citeauthoryear{Liu \bgroup et al\mbox.\egroup
  }{2017}]{liu2017stein}
Liu, Y.; Ramachandran, P.; Liu, Q.; and Peng, J.
\newblock 2017.
\newblock Stein variational policy gradient.
\newblock {\em arXiv preprint arXiv:1704.02399}.

\bibitem[\protect\citeauthoryear{Liu, Lee, and
  Jordan}{2016}]{liu2016kernelized}
Liu, Q.; Lee, J.~D.; and Jordan, M.~I.
\newblock 2016.
\newblock A kernelized stein discrepancy for goodness-of-fit tests.
\newblock In {\em Proceedings of the International Conference on Machine
  Learning (ICML)}.

\bibitem[\protect\citeauthoryear{Liu, Zhu, and Song}{2016}]{liu2016stochastic}
Liu, C.; Zhu, J.; and Song, Y.
\newblock 2016.
\newblock Stochastic gradient geodesic mcmc methods.
\newblock In {\em Advances In Neural Information Processing Systems},
  3009--3017.

\bibitem[\protect\citeauthoryear{Ma, Chen, and Fox}{2015}]{ma2015complete}
Ma, Y.-A.; Chen, T.; and Fox, E.
\newblock 2015.
\newblock A complete recipe for stochastic gradient mcmc.
\newblock In {\em Advances in Neural Information Processing Systems},
  2917--2925.

\bibitem[\protect\citeauthoryear{Mardia and
  Jupp}{2000}]{mardia2000distributions}
Mardia, K.~V., and Jupp, P.~E.
\newblock 2000.
\newblock Distributions on spheres.
\newblock {\em Directional Statistics}  159--192.

\bibitem[\protect\citeauthoryear{Micchelli and
  Pontil}{2005}]{micchelli2005learning}
Micchelli, C.~A., and Pontil, M.
\newblock 2005.
\newblock On learning vector-valued functions.
\newblock {\em Neural computation} 17(1):177--204.

\bibitem[\protect\citeauthoryear{Mika \bgroup et al\mbox.\egroup
  }{1999}]{mika1999fisher}
Mika, S.; Ratsch, G.; Weston, J.; Scholkopf, B.; and Mullers, K.-R.
\newblock 1999.
\newblock Fisher discriminant analysis with kernels.
\newblock In {\em Neural Networks for Signal Processing IX, 1999. Proceedings
  of the 1999 IEEE Signal Processing Society Workshop.},  41--48.
\newblock IEEE.

\bibitem[\protect\citeauthoryear{Pu \bgroup et al\mbox.\egroup
  }{2017}]{pu2017stein}
Pu, Y.; Gan, Z.; Henao, R.; Li, C.; Han, S.; and Carin, L.
\newblock 2017.
\newblock Stein variational autoencoder.
\newblock {\em arXiv preprint arXiv:1704.05155}.

\bibitem[\protect\citeauthoryear{Reisinger \bgroup et al\mbox.\egroup
  }{2010}]{reisinger2010spherical}
Reisinger, J.; Waters, A.; Silverthorn, B.; and Mooney, R.~J.
\newblock 2010.
\newblock Spherical topic models.
\newblock In {\em Proceedings of the 27th International Conference on Machine
  Learning (ICML-10)},  903--910.

\bibitem[\protect\citeauthoryear{Rezende and
  Mohamed}{2015}]{rezende2015variational}
Rezende, D., and Mohamed, S.
\newblock 2015.
\newblock Variational inference with normalizing flows.
\newblock In {\em Proceedings of The 32nd International Conference on Machine
  Learning},  1530--1538.

\bibitem[\protect\citeauthoryear{Romano}{2007}]{romano2007continuum}
Romano, G.
\newblock 2007.
\newblock Continuum mechanics on manifolds.
\newblock {\em Lecture notes University of Naples Federico II, Naples, Italy}
  1--695.

\bibitem[\protect\citeauthoryear{Sherman and
  Morrison}{1950}]{sherman1950adjustment}
Sherman, J., and Morrison, W.~J.
\newblock 1950.
\newblock Adjustment of an inverse matrix corresponding to a change in one
  element of a given matrix.
\newblock {\em Annals of Mathematical Statistics} 21(1):124--127.

\bibitem[\protect\citeauthoryear{Song and Zhu}{2016}]{song2016bayesian}
Song, Y., and Zhu, J.
\newblock 2016.
\newblock Bayesian matrix completion via adaptive relaxed spectral
  regularization.
\newblock In {\em The 30th AAAI Conference on Artificial Intelligence
  (AAAI-16)}.

\bibitem[\protect\citeauthoryear{Steinwart and
  Christmann}{2008}]{steinwart2008support}
Steinwart, I., and Christmann, A.
\newblock 2008.
\newblock {\em Support vector machines}.
\newblock Springer Science \& Business Media.

\bibitem[\protect\citeauthoryear{Wang and Liu}{2016}]{wang2016learning}
Wang, D., and Liu, Q.
\newblock 2016.
\newblock Learning to draw samples: With application to amortized mle for
  generative adversarial learning.
\newblock {\em arXiv preprint arXiv:1611.01722}.

\bibitem[\protect\citeauthoryear{Zhang, Reddi, and
  Sra}{2016}]{zhang2016riemannian}
Zhang, H.; Reddi, S.~J.; and Sra, S.
\newblock 2016.
\newblock Riemannian svrg: Fast stochastic optimization on riemannian
  manifolds.
\newblock In {\em Advances in Neural Information Processing Systems},
  4592--4600.

\bibitem[\protect\citeauthoryear{Zhou}{2008}]{zhou2008derivative}
Zhou, D.-X.
\newblock 2008.
\newblock Derivative reproducing properties for kernel methods in learning
  theory.
\newblock {\em Journal of computational and Applied Mathematics}
  220(1):456--463.

\end{thebibliography}

\newpage
\section*{Appendix}

\subsection*{A1. Proof of Lemma~1 (Continuity Equation on Riemann Manifold)}
Let $F_{(\cdot)}(\cdot)$ be the flow of $X$.
$\forall U\subset\mf$ compact, consider the integral $\int_{F_t(U)}p_t \mu_g$.
Since a particle in $U$ at time $0$ will always in $F_t(U)$ at time $t$ and vice versa, the integral, i.e. the portion of particles in $F_t(U)$ at time $t$, is equal to the portion of particles in $U$ at time $0$ for any time $t$.
So it is a constant.
Reynolds transport theorem gives
\begin{align*}
  0 = \frac{\ud}{\ud t} \int_{F_t(U)} p_t \mu_g = \int_{F_t(U)} \left( \frac{\partial p_t}{\partial t} + \div(p_t X) \right) \mu_g
\end{align*}
for any $U$ and $t$, so the integrand must be zero and we derived the conclusion.

\subsection*{A2. Well-definedness of KL-divergence on Riemann Manifold}
We define the KL-divergence between two distributions on $\mf$ by their \pdf $q^{\mu}$ and $p^{\mu}$ \wrt volume form $\mu$ as:
\begin{align*}
  \kl(q||p) \defas \int_{\mf} q^{\mu} \log(q^{\mu}/p^{\mu}) \mu.
\end{align*}
To make this notion well-defined, we need to show that the right hand side of the definition is invariant of $\mu$.
Let $\omega$ be another volume form.
Since $\forall A\in\mf$, $\mu(A)$ and $\omega(A)$ lie on the same 1-dimensional linear space (the space of $m$-forms at $A$), we have $\alpha(A)\in\re^+$ s.t. $\omega(A) = \alpha(A)\mu(A)$.
Such a construction gives a smooth function $\alpha:\mf\to\re^+$.
By the definition of p.d.f., $q^{\omega} = q^{\mu}/\alpha$.
So $\int_{\mf} q^{\omega} \log(q^{\omega}/p^{\omega}) \omega = \int_{\mf} q^{\mu} \log(q^{\mu}/p^{\mu}) \mu$, which indicates that the integral is independent of the chosen volume form.

\subsection*{A3. Proof of Theorem~2}
To formally prove Theorem~2, we first deduce a lemma, which gives the \pdf of the distribution transformed by a diffeomorphism on $\mf$ (an invertible smooth transformation on $\mf$).
\begin{lem}[Transformed p.d.f.]
  \label{thm:transpdf}
  Let $\phi$ be an orientation-preserving diffeomorphism on $\mf$, and $p$ the \pdf of a distribution on $\mf$. Denote $p_{[\phi]}$ as the \pdf of the distribution of the $\phi$-transformed random variable from the one obeying $p$, i.e. the transformed p.d.f. Then in any local coordinate system (c.s.) $(U,\Phi)$,
  \begin{align}
	p_{[\phi]} = \frac{\big(p\sqrt{|G|}\big)\circ\phi^{-1}}{\sqrt{|G|}} \left| \jac\phi^{-1} \right|,
	\label{eqn:transpdf}
  \end{align}
  where $G$ is the Riemann metric tensor in $(U,\Phi)$ and $|G|$ is its determinant, and $\jac\phi^{-1}$ is the Jacobian determinant of $\Phi\circ\phi^{-1}\circ\Phi^{-1}:\re^m\to\re^m$.
  The right hand side is coordinate invariant.
\end{lem}
\begin{proof}
  Let $U$ be a compact subset of $\mf$, and $(V,\Phi), V\subset U$ be a local c.s. of $U$ with coordinate chart $\{x^i\}_{i=1}^m$.
  On one hand, due to the definition of $p_{[\phi]}$, we have $\prob_{p}(U) = \prob_{p_{[\phi]}}(\phi(U))$.
  On the other hand, we can invoke the theorem of global change of variables on manifold (\citep{abraham2012manifolds}, Theorem~8.1.7), which gives $\prob_{p}(U) = $
  \begin{align}
	\label{eqn:changeofvar1} &\int_U p\mu_g \!=\!\! \int_{\!\phi(U)} \!\!{\phi^{-1}}^* \!(p\mu_g) =\!\! \int_{\!\phi(U)} \!(p\circ\!\phi^{-1}) {\phi^{-1}}^*\!(\mu_g) \\
	\notag =& \int_{\phi(U)} (p\circ\phi^{-1}) (\sqrt{|G|}\circ\phi^{-1}) |\jac \phi^{-1}| \ud x^1\wedge\dots\wedge\ud x^m \\
	\label{eqn:changeofvar2} =& \int_{\phi(U)} \frac{\big(p\sqrt{|G|}\big)\circ\phi^{-1}}{\sqrt{|G|}} |\jac \phi^{-1}| \mu_g
  \end{align}
  $= \prob_{\frac{\left(p\sqrt{|G|}\right)\circ\phi^{-1}}{\sqrt{|G|}} |\jac \phi^{-1}|} (\phi(U))$, where ${\phi^{-1}}^*(\cdot)$ is the pull-back of $\phi^{-1}$ on the $m$-forms on $\mf$.
  Combining both hands and noting the arbitrariness of $U$, we get the desired conclusion.
\end{proof}

Let $F_{(\cdot)}(\cdot)$ be the flow of $X$.
For any evolving distribution $p_t$ under dynamics $X$, by its definition, we have $p_t = {p_0}_{[F_t]}$.
Due to the property of flow that for any $s,t\in\re$, $F_{s+t} = F_s\circ F_t = F_t\circ F_s$, we have $p_{s+t} = {p_0}_{[F_{s+t}]} = {p_0}_{[F_s\circ F_t]} = ({p_0}_{[F_s]})_{[F_t]} = (p_s)_{[F_t]}$.

Now, for a fixed $t_0\in\re$, we let $p_t$ be the evolving distribution under $X$ that satisfies $p_{t_0}=p$, the target distribution.
For sufficiently small $t>0$, $F_{t}(\cdot)$ is a diffeomorphism on $\mf$.
Equipped with all these knowledge, we begin the final deduction:
\begin{align*}
  &-\left.\frac{\ud}{\ud t}\right|_{t=t_0}\kl( q_t||p ) = -\left.\frac{\ud}{\ud t}\right|_{t=0} \int_{\mf} q_{t_0+t} \log\frac{q_{t_0+t}}{p_{t_0}} \mu_g
  \intertext{(Treat $q_{t_0+t}$ as $(q_{t_0})_{[F_t]}$ and apply Eqn.~(\ref{eqn:transpdf}))}
  =& -\left.\frac{\ud}{\ud t}\right|_{t=0} \int_{\mf} \frac{\big(q_{t_0}\sqrt{|G|}\big)\circ F^{-1}_t}{\sqrt{|G|}} \left| \jac F^{-1}_t \right| \\
  &\cdot \left( \log \frac{\big(q_{t_0}\sqrt{|G|}\big)\circ F^{-1}_t}{\sqrt{|G|}} + \log \left| \jac F^{-1}_t \right| - \log p_{t_0} \right) \mu_g
  \intertext{(Apply $F_t^{-1}$ to the entire integral and invoke the theorem of global change of variables Eqn.~(\ref{eqn:changeofvar1}))}
  =& -\left.\frac{\ud}{\ud t}\right|_{t=0} \int_{F_t^{-1}(\mf)} \Bigg(\Bigg[ \frac{\big(q_{t_0}\sqrt{|G|}\big)\circ F^{-1}_t}{\sqrt{|G|}} \left| \jac F^{-1}_t \right| \\
  &\!\cdot \!\left(\! \log\! \frac{\big(q_{t_0}\!\sqrt{|G|}\big)\!\!\circ\! F^{-1}_t}{\sqrt{|G|}} + \log \!\left| \jac F^{-1}_t \!\right| \!-\! \log p_{t_0} \!\right)\!\!\Bigg]\!\!\circ\! F_t\!\Bigg) F_t^*(\mu_g)
  \intertext{($F_t^{-1}(\mf)=\mf$ since $F_t^{-1}$ is a diffeomorphism on $\mf$. $|\jac F_t^{-1}|\circ F_t = |\jac F_t|^{-1}$. See Eqn.~(\ref{eqn:changeofvar2}) for the expression of $F_t^*(\mu_g)$, the pull-back of $F_t$ on $\mu_g$)}
  =& -\left.\frac{\ud}{\ud t}\right|_{t=0} \int_{\mf} \frac{q_{t_0}\sqrt{|G|}}{\sqrt{|G|}\circ F_t} \left| \jac F_t \right|^{-1} \cdot \Bigg( \log \frac{q_{t_0}\sqrt{|G|}}{\sqrt{|G|}\circ F_t} \\
  &- \log \left| \jac F_t \right| - \log (p_{t_0}\!\circ F_t) \Bigg) \cdot \frac{\sqrt{|G|}\circ F_t}{\sqrt{|G|}} \left|\jac F_t\right| \mu_g \displaybreak
  \intertext{(Rearange terms)}
  =& -\!\left.\frac{\ud}{\ud t}\right|_{t=0} \!\int_{\mf} \!q_{t_0} \!\!\left[ \log q_{t_0} \!- \log\!\left(\! \frac{\big(p_{t_0}\sqrt{|G|}\big)\!\circ\! F_t}{\sqrt{|G|}} \left|\jac F_t\right| \!\right) \!\right] \!\mu_g
  \intertext{(Note the property of flow: $F_t = F_{-t}^{-1}$. Treat $p_{t_0-t}$ as $(p_{t_0})_{[F_{-t}]}$ and apply Eqn.~(\ref{eqn:transpdf}) inversely)}
  =& -\left.\frac{\ud}{\ud t}\right|_{t=0} \int_{\mf} q_{t_0} \left[ \log q_{t_0} - \log p_{t_0-t} \right] \mu_g
  &\intertext{($\mf$ is unchanged over time $t$ (otherwise an integral over the boundary would appear))}
  =& \! \int_{\mf} \! q_{t_0}\! \! \left.\frac{\partial}{\partial t} (\log p_{t_0-t})\right|_{t=0} \! \mu_g = -\! \int_{\mf} \! q_{t_0}\!\! \left.\frac{\partial}{\partial t} (\log p_{t_0+t})\right|_{t=0} \! \mu_g
  \intertext{(Refer to Eqn.~(3))}
  =& \int_{\mf} (q_{t_0}/p_{t_0}) \div(p_{t_0} X) \mu_g = \expect_{q_{t_0}} [\div(p_{t_0} X) / p_{t_0}]
  \intertext{(Property of divergence)}
  =& \expect_{q_{t_0}} \big[ X[\log p_{t_0}] + \div(X) \big].
\end{align*}
Due to the arbitrariness of $t_0$, we get the desired conclusion and complete the proof.

\subsection*{A4. Condition for Stein's Identity (Stein Class)}

Now we derive the condition for Stein's identity to hold.
We require $\expect_p[\div(pX)/p] = 0$, which is
\begin{align*}
  &\int_{\mf} \div(pX) \mu_g = \int_{\partial\mf} \mathrm{i}_{(pX)} \mu_g \\
  =& \sum_{i=1}^m \int_{\partial\mf} p\sqrt{|G|} (-1)^{i+1} X^i \bigwedge\! \ud x^{\neg i},
\end{align*}
where the first equality holds due to Gauss' theorem~(\citep{abraham2012manifolds}, Theorem~8.2.9), $\partial\mf$ is the boundary of $\mf$, $\mathrm{i}_{X}:A^{k}(\mf) \to A^{k-1}(\mf)$ is the interior product or contraction, $(\mathrm{i}_{X}\omega)(A)[v_1,\dots,v_{k-1}] = \omega(A)[X(A), v_1, \dots, v_{k-1}]$, $X^i$ is the $i$-th component of $X$ under the natural basis of some local c.s., $\bigwedge \ud x^{\neg i} \defas \ud x^1 \wedge \dots \wedge \ud x^{i-1} \wedge \ud x^{i+1} \wedge \dots \wedge \ud x^m$ with ``$\wedge$'' the wedge product (exterior product).

For manifolds like spheres, $\partial\mf$ is empty and the above integral is always zero, so the Stein class is $\tg(\mf)$.
If $\partial\mf$ is not empty, by its definition, around any point on $\partial\mf$ there exists a c.s. $(V, \Psi)$ with coordinate chart $(y^1, \dots, y^m)$ such that $\forall A\in\partial\mf \cap V, y^m(A) = 0$.
Thus $\ud y^m = 0$ and $(\partial\mf\cap V, \tilde\Psi=(\Psi^1, \dots, \Psi^{m-1}))$ is a local c.s. of $\partial\mf$.
Then the condition for Stein's identity to hold becomes
\begin{align*}
  \int_{\partial\mf} p\tilde X^m \sqrt{|\tilde G|} \ud y^1 \wedge \dots \wedge \ud y^{m-1} = 0,
\end{align*}
where $\tilde G$ is the Riemann metric tensor in $(\partial\mf\cap V, \tilde\Psi)$, and $\tilde X^{m}$ is the $m$-th component of $X$ in $(\partial\mf\cap V, \tilde\Psi)$.

For the case where $\mf$ is a compact subset of Euclidean space $\re^m$, around any point $A$ on the boundary $\partial\mf$, we take $(V, \Psi)$ such that $y^m=0$ and the natural basis $\{\partial_i|\partial_i \defas \frac{\partial}{\partial y^i}, i=1,\dots,m\}$ is orthonormal.
Then $|\tilde G(A)| = 1$ and $\partial_m$ is perpendicular to the span of $\{\partial_1, \dots, \partial_{m-1}\}$, which is the tangent space of $\partial\mf$ at $A$.
So $\partial_m$ is the unit normal $\vec n$ to $\partial\mf$, and $\tilde X^m$ is the component of $X$ along the normal direction, i.e. $\tilde X^m = X\cdot\vec n$.
Denote the volume form $\ud y^1 \wedge \dots \wedge \ud y^{m-1}$ on $\partial\mf$ as $\ud S$, then the condition for Stein's identity is $\int_{\partial\mf} p X \cdot \vec n \ud S$, which meets the conclusion in~\cite{liu2016stein}.
We provide a generalization of the conclusion to general Riemann manifold.

\subsection*{A5. Proof of Theorem~4}

For any $X\in\subtg$, let $f = \iota^{-1}(X)$ ($\iota$ is defined in the proof of Lemma~3), i.e. the only element in $\hilb_K$ such that $X = \grad f$.
Then in any c.s., $X = g^{ij}\partial_i f \partial_j$, and we have
\begin{align*}
  &\obj(X) \defas \expect_{q} \left[ X[\log p] + \div(X) \right] \\
  =& \expect_q \left[ X^j\partial_j\log(p\sqrt{|G|}) + \partial_j X^j \right] \\
  =& \expect_q \left[ g^{ij}\partial_i f \partial_j\log(p\sqrt{|G|}) + \partial_j(g^{ij}\partial_i f) \right] \\
  =& \expect_q \left[ \left( g^{ij}\partial_j\log(p\sqrt{|G|}) + \partial_j g^{ij} \right) \partial_i f + g^{ij}\partial_i\partial_i f \right].
\end{align*}
Now we invoke the conclusions of~\citet{zhou2008derivative} that $\partial_i K(A,\cdot), \partial_i\partial_j K(A,\cdot) \in \hilb_K$, and for any $f\in\hilb_K$, $\langle f(\cdot),\partial_i K(A,\cdot) \rangle_{\hilb_K} = \partial_i f(A)$, $\langle f(\cdot),\partial_i\partial_j K(A,\cdot) \rangle_{\hilb_K} = \partial_i\partial_j f(A)$:
\begin{align*}
  \obj(X) =& \expect_q \Big[ \left( g^{ij}\partial_j\log(p\sqrt{|G|}) + \partial_j g^{ij} \right) \left< f(\cdot),\partial_i K(A,\cdot) \right>_{\hilb_K} \\
  & + g^{ij} \left< f(\cdot),\partial_i\partial_j K(A,\cdot) \right>_{\hilb_K} \Big] \\
  =& \expect_q \Big[ \Big< f(\cdot), \left( g^{ij}\partial_j\log(p\sqrt{|G|}) + \partial_j g^{ij} \right) \partial_i K(A,\cdot) \\
  & + g^{ij} \partial_i\partial_j K(A,\cdot) \Big>_{\hilb_K} \Big] \\
  =& \Big< f(\cdot), \expect_q \Big[ \left( g^{ij}\partial_j\log(p\sqrt{|G|}) + \partial_j g^{ij} \right) \partial_i K(A,\cdot) \\
  & + g^{ij} \partial_i\partial_j K(A,\cdot) \Big] \Big>_{\hilb_K},
\end{align*}
where all the functions, differentiations and expectations are with argument $A$, if not specified.
Define
\begin{align*}
  \hat f(\cdot) =& \expect_q \Big[ \left( g^{ij}\partial_j\log(p\sqrt{|G|}) + \partial_j g^{ij} \right) \partial_i K(A,\cdot) \\
  & + g^{ij} \partial_i\partial_j K(A,\cdot) \Big] \\
  =& \expect_q \Big[ g^{ij} \partial_j\log(p\sqrt{|G|}) \partial_i K(A,\cdot) \\
  & + \partial_j \big( \sqrt{|G|} g^{ij} \partial_i K(A,\cdot) \big) / \sqrt{|G|} \Big] \\
  =& \expect_q \Big[ g^{ij} \partial_j\log(p\sqrt{|G|}) \partial_i K(A,\cdot) + \Delta K(A,\cdot) \Big],
\end{align*}
we have $\obj(X) = \langle f(\cdot),\hat f(\cdot) \rangle_{\hilb_K}$, and by the isometric isomorphism between $\hilb_K$ and $\subtg$, we have $\obj(X) = \langle \grad f,\grad \hat f \rangle_{\subtg} = \langle X,\hat X \rangle_{\subtg}$.


\subsection*{A6 Expressions in the Isometrically Embedded Space}
In this part of appendix we express the functional gradient in the isometrically embedded space, for general Riemann manifolds and two specific Riemann manifolds.

\subsubsection*{A6.1 For General Riemann Manifolds (Proposition~6)}
Let $\Xi$ be an isometric embedding of $\mf$ into $(\re^n,\{y^{\alpha}\}_{\alpha=1}^n)$.
For a coordinate system (c.s.) $(U,\Phi)$ of $\mf$ with coordinate chart $\{x^i\}_{i=1}^m$, define $\xi \defas \Xi\circ\Phi^{-1}$.
We first develop a key tool.
Let $h:\Xi(\mf)\to\re$ be a smooth function on the embedded manifold.
In $(U,\Phi)$ we define $f \defas h\circ\xi: U\to\re$ as a smooth function on an open subset of $\re^m$.
By the chain rule of derivative, we have
\begin{align*}
  \partial_i f = \partial_{\alpha} h \frac{\partial y^{\alpha}}{\partial x^{i}} = M\trs \nabla h,
\end{align*}
where $M\in\re^{n\times m}: M_{\alpha i} = \frac{\partial y^{\alpha}}{\partial x^i}$, and $\nabla h$ is the usual gradient of $h$ as a function on $\re^n$.
For isometric embedding, we have $g_{ij} = \sum_{\alpha=1}^n \frac{\partial y^{\alpha}}{\partial x^i}\frac{\partial y^{\alpha}}{\partial x^j}$, or in matrix form $G = M\trs M$.

From Eqn.~(5), we know that $\hat f' = \expect_q\big[ f_1 + f_2 \big]$ where $f_1 = (\grad K)[\log p]$ and $f_2 = \Delta K$.
Then in any c.s. of $\mf$,
\begin{align*}
  f_1 =& g^{ij} (\partial_i\log p) (\partial_j K) \\
  =& g^{ij}\frac{\partial y^{\alpha}}{\partial x^i} (\partial_{\alpha}\log p) \frac{\partial y^{\beta}}{\partial x^j} (\partial_{\beta} K) \\
  =& (\nabla\log p)\trs (M G^{-1} M\trs) \nabla K, \\
  f_2 =& g^{ij}(\partial_i K) (\partial_j\log\sqrt{|G|}) + \partial_i(g^{ij}\partial_j K) \\
  =& (\nabla\log\sqrt{|G|})\trs (M G^{-1} M\trs) \nabla K + \frac{\partial y^{\alpha}}{\partial x^i} \partial_{\alpha} (g^{ij} \frac{\partial y^{\beta}}{\partial x^j} \partial_{\beta}K) \\
  =& (\nabla\log\sqrt{|G|})\trs (M G^{-1} M\trs) \nabla K \\
  &+ (M\trs\nabla)\trs(G^{-1} M\trs \nabla K) \\
  =& (\nabla\log\sqrt{|G|})\trs (M G^{-1} M\trs) \nabla K \\
  & + \Big( (M\trs\nabla)\trs (G^{-1}M\trs) \Big) \nabla K \\
  & + \tr\Big( (\nabla\nabla\trs K) (M G^{-1} M\trs) \Big).
\end{align*}

To further simplify the expression, we mention it here that the operator $M G^{-1} M\trs = M (M\trs M)^{-1} M\trs$ is the orthogonal projection onto the column space of $M$, which is the tangent space of the embedded manifold.
With $N\in\re^{n\times(n-m)}$ consisting of a set of orthonormal basis of the orthogonal complement of the tangent space, we can express the operator as $(I_n-NN\trs)$.
Details are presented in \citet{byrne2013geodesic} or Appendix A.2 of \citet{liu2016stochastic}.
The advantage of using $N$ instead of $M$ is that it is independent of c.s. of $\mf$, so we do not need to choose a set of c.s. covering $\mf$ and conduct calculation in each c.s.
Additionally, it is usually easier to find, and the expression with $N$ is more computationally economic.
With this replacement, we have
\begin{align*}
  f_1 \!+\! f_2 =& (\nabla\log p\sqrt{|G|})\trs (M G^{-1} M\trs) \nabla K \\
  &+ \Big( (M\trs\nabla)\trs (G^{-1}M\trs) \Big) \nabla K \\
  &+ \tr\Big( (\nabla\nabla\trs K) (M G^{-1} M\trs) \Big) \\
  =& (\nabla\log p\sqrt{|G|})\trs (I_n - NN\trs) \nabla K \\
  &+ \Big( (M\trs\nabla)\trs (G^{-1}M\trs) \Big) \nabla K \\
  &+ \tr\Big( (\nabla\nabla\trs K) - (\nabla\nabla\trs K)NN\trs \Big) \\
  =& (\nabla\log p\sqrt{|G|})\trs (I_n - NN\trs) \nabla K \\
  &+ \Big( (M\trs\nabla)\trs (G^{-1}M\trs) \Big) \nabla K \\
  &+ \nabla\trs\nabla K - \tr\Big( N\trs(\nabla\nabla\trs K)N \Big).
\end{align*}

Finally, $\hat X = \grad \hat f = g^{ij}\partial_i \hat f \partial_j = g^{ij} \frac{\partial y^{\alpha}}{\partial x^i} \partial_{\alpha} \hat f \frac{\partial y^{\beta}}{\partial x^j} \partial_{\beta} = M G^{-1} M \nabla\hat f = (I_n - NN\trs)\nabla \hat f$, which finishes the derivation.

Note that $M$ and $G$ depend on the choice of c.s. of $\mf$.
Note also that the parametric form of $\Xi^{-1}$ and $\xi^{-1}$ may not be unique (e.g. $\Xi^{-1}(y)=y$ and $\Xi^{-1}(y)=y+(1-y\trs y)$ are both valid on $\Xi(\sph^{n-1})$, but they give different gradients).
Nevertheless, since $\hat f'$ is already a well-defined smooth function on $\mf$ due to Eqn.~(5), its expression in the embedded space \wrt any c.s. and any parametric form of $\Xi^{-1}$ and $\xi^{-1}$ should give the same result.
We introduce $N$ in hope to explicitly express this independence, and we succeed for $\hat X'$ given $\hat f'$.
For $\hat f'$, it is still a future work to make its expression explicitly independent of c.s. of $\mf$ and parametric form of $\Xi^{-1}$ and $\xi^{-1}$.

\subsubsection*{A6.2 For Hyperspheres (Proposition~7)}
Let $\sph^{n-1}$ be isometrically embedded in $\re^n$ via $\Xi:y\mapsto y$ the identity mapping.
We select the c.s. $(U,\Phi)$ as the upper semi-hypersphere: $U\defas \{y\in\re^n| y\trs y = 1, y_n > 0\}$, $\Phi: y\mapsto (y_1,\dots,y_{n-1})\trs\in\re^{n-1}$.
Then we have $\Omega \defas \Phi(U) = \{x\in\re^{n-1} | x\trs x < 1\}$, and $\xi:\Omega\to\re^n, x\mapsto (x_1,\dots,x_{n-1}, \sqrt{1-x\trs x})\trs$.
Furthermore,
\begin{equation*}
  M = \left( \begin{array}{c}
	I_{n-1} \\
	-\frac{x\trs}{\sqrt{1-x\trs x}}
  \end{array}\right),
\end{equation*}
and $G = I_{n-1} + \frac{xx\trs}{1-x\trs x}$, $G^{-1} = I_{n-1} - xx\trs$, $|G| = \frac{1}{1-x\trs x}$.
The tangent space of $\Xi(\sph^{n-1})$ at $y\in\re^n$ is a plane perpendicular to the direction of $y$, thus the orthogonal complement of the tangent space is the linear span of $y$, which indicates that $N=y$.
Plugging in all these quantities in Eqn.~(7), we can derive the result of Eqn.~(8).

\subsubsection*{A6.3 For the Product Manifold of Hyperspheres}
To fit the inference task of Spherical Admixture Model \cite{reisinger2010spherical} (SAM), we need to further specify the manifold as the product manifold of hyperspheres, $(\sph^{n-1})^P$.
Let $(\mf)^P$ be a general product manifold.
For any point $A=(A_{(1)}, \dots, A_{(P)})\in(\mf)^P$, $(\bigotimes_{k=1}^P U_{(k)}, \bigotimes_{k=1}^P \{x_{(k)}^{i_{(k)}}\}_{i_{(k)}=1}^{n-1})$ is a local c.s., where each $(U_{(k)},\{x_{(k)}^{i_{(k)}}\}_{i_{(k)}=1}^{n-1}$ is a local c.s. of $\mf_{(k)}$ around $A_{(k)}$.
In this c.s., $\{\partial_{(k),i_{(k)}} | k=1,\dots,P, i_{(k)} = 1,\dots,n-1\}$ is the natural basis, and the Riemann structure in the tangent space is defined by direct product of inner product space: $g_{(k,\ell),i_{(k)},j_{(\ell)}} = \delta_{k\ell} g_{i_{(k)},j_{(\ell)}}$.
By this construction, one can derive the expressions for the gradient of a smooth function $f\in\cont^{\infty}((\mf)^P)$ and the divergence of a vector field $X=\sum_{k=1}^P X_{(k)}^{i_{(k)}} \partial_{(k),i_{(k)}} \in \tg((\mf)^P)$:
$\grad f = \sum_{k=1}^P g^{i_{(k)}j_{(k)}}_{(k)} \partial_{(k),i_{(k)}}f \partial_{(k),j_{(k)}}$,
$\div(X) = \sum_{k=1}^P \left( \partial_{(k), i_{(k)}} X^{i_{(k)}}_{(k)} + X^{i_{(k)}}_{(k)} \partial_{(k),i_{(k)}} \log\sqrt{ |G_{(k)}| } \right)$, as well as the Beltrami-Laplacian $\Delta f$.

For $y=(y_{(1)}, \dots, y_{(P)})\in (\sph^{n-1})^P$ with each $y_{(k)}\in\sph^{n-1}$, and kernel $K(y,y') = \prod_{k=1}^P K_{(k)}(y_{(k)}, y_{(k)}')$, we have the following result:
\begin{prop}
  \label{thm:prodsph}
  $\hat X'_{(\ell)} = (I_d - y_{(\ell)} {y_{(\ell)}'}\trs)\nabla_{(\ell)}' \hat f'$,
  \begin{align}
    \label{eqn:prodsph}
    \hat f' =& \expect_q \Big[ K \sum_{k=1}^{P} \Big[ (\nabla_{\!(k)}\!\log p\big)\!\trs (\nabla_{\!(k)} \!\log K_{(k)}) + \nonumber\\
      & \nabla_{\!(k)}\trs\!\nabla_{\!(k)} \log K_{(k)} - y_{(k)}\trs\big( \nabla_{\!(k)}\!\nabla_{\!(k)}\trs K_{(k)} \big) y_{(k)} \nonumber\\
      & + \big\|\nabla_{\!(k)} \!\log K_{(k)} \big\|^2 - (y_{(k)}\trs\!\nabla_{\!(k)} \!\log K_{(k)})^2 \nonumber\\
      & - (y_{(k)}\trs\!\nabla_{\!(k)}\!\log p + n - 1) y_{(k)}\trs \!\nabla_{\!(k)} \!\log K_{(k)} \Big] \Big].
  \end{align}
\end{prop}
This proposition directly constructs the algorithm of RSVGD for the inference task of SAM, where each $y_{(k)}$ is a topic lying on a hypersphere.

\subsubsection*{A7 Implementation of RSVGD for Bayesian Logistic Regression}
From the model description in the main context, we have
\begin{eqnarray*}
  & \mbox{log-prior:} & \log p_0(w) = -\frac{w\trs w}{2\alpha} + \const, \\
  & \mbox{log-likelihood:} & \log p(\{y_d\}|w, \{x_d\}) \\
  & & \!\!\!\!\!\!\! = \sum_{d=1}^D \left( y_d w\trs x_d - \log(1+e^{w\trs x_d}) \right) + \const, \\
  & \mbox{log-posterior:} & \log p(w | \{y_d\}, \{x_d\}) = -\frac{w\trs w}{2\alpha} \\
  & & + \sum_{d=1}^D \left( y_d w\trs x_d - \log(1+e^{w\trs x_d}) \right) + \const.
\end{eqnarray*}
So we have the gradient of the target density
\begin{align*}
  \nabla \log p(w|\{y_d\}, \{x_d\}) = -\frac{1}{\alpha} w + \sum_{d=1}^D \left( y_d - s(w\trs x_d) \right) x_d,
\end{align*}
and the Riemann metric tensor
\begin{align*}
  G(w) =& \fisher\big(p(\{y_d\}|w, \{x_d\})\big) - \nabla\nabla\trs \log p_0(w) \\
  =& \expect_{p(\{y_d\}|w, \{x_d\})} \big[\big(\nabla \log p(\{y_d\}|w, \{x_d\})\big) \\
  & \phantom{\expect_{p(\{y_d\}|w, \{x_d\})} \big[} \big(\nabla \log p(\{y_d\}|w, \{x_d\})\big)\trs \big] \\
  & - \nabla\nabla\trs \log p_0(w) \\
  =& \sum_{d=1}^D c_d x_d x_d\trs + \frac{1}{\alpha} I_m,
\end{align*}
where $\fisher(\cdot)$ is the Fisher information of a distribution, and $c_d = s(w\trs x_d) (1 - s(w\trs x_d))$.
For $G^{-1}$, direct numerical inversion is applicable, with time complexity $\order(m^3)$.
Another method, with time complexity $\order(m^2 D)$, can be derived by iteratively applying the Sherman-Morrison formula \cite{sherman1950adjustment}:
\begin{align*}
  & G^{-1}_d = G^{-1}_{d-1} - \frac{c_d (G^{-1}_{d-1} x_d) (G^{-1}_{d-1} x_d)\trs}{1 + c_d x_d\trs G^{-1}_{d-1} x_d}, \\
  & G^{-1} = G^{-1}_D, G^{-1}_0 = \alpha I_m.
\end{align*}
For small datasets, or for mini-batch of data, this implementation would be advantageous.
But in our experiments we found that direct inversion is still more efficient.

To continue, we first note $\partial_i G \defas \partial_{w_i} G = \sum_{d=1}^D f_d x_{di} x_d x_d\trs$, where $f_d = \frac{1-e^{w\trs x_d}}{1+e^{w\trs x_d}} c_d$.
Note also that $\partial_i G_{jk} = \sum_{d=1}^D f_d x_{di} x_{dj} x_{dk}$, so the indices $i,j,k$ are completely permutable.
Particularly, $\partial_i G_{jk} = \partial_j G_{ik}$.
For the gradient of the log-determinant,
\begin{align*}
  \partial_i \log |G(w)| = \tr(G^{-1} \partial_i G) = \sum_{d=1}^D f_d (x_d\trs G^{-1} x_d) x_{di},
\end{align*}
and for the gradient of the inverse matrix,
\begin{align*}
  & \sum_{j=1}^m \partial_j G^{-1}_{ij}(w) = - G^{-1}_{(i,:)} \sum_{j=1}^m (\partial_j G) G^{-1}_{(:,j)} \\
  =& - \sum_{k=1}^m G^{-1}_{(i,k)} \sum_{j=1}^m \sum_{\ell=1}^m (\partial_j G)_{(k,\ell)} G^{-1}_{(\ell, j)} \\
  =& - \sum_{k=1}^m G^{-1}_{(i,k)} \sum_{j=1}^m \sum_{\ell=1}^m (\partial_k G)_{(j,\ell)} G^{-1}_{(\ell, j)} \\
  =& - \sum_{k=1}^m G^{-1}_{(i,k)} \tr\big( (\partial_k G) G^{-1} \big) \\
  =& - G^{-1}_{(i,:)} \nabla \log |G(w)|.
\end{align*}
Now all the quantities needed for RSVGD (Eqn.~(6)) are derived.

\end{document}